\documentclass[11pt]{article}

\usepackage[margin=1in]{geometry}
\usepackage[T1]{fontenc}

\usepackage{amsmath,amssymb,amsthm}
\usepackage{graphicx}
\usepackage{xcolor}
\definecolor{mygreen}{RGB}{0,130,0}
\usepackage{hyperref}
\hypersetup{
  colorlinks=true,
  linkcolor=red,
  citecolor=mygreen,
  urlcolor=blue
}
\usepackage{float}
\usepackage{algorithm}
\usepackage{algpseudocode}
\usepackage{tikz}
\usetikzlibrary{positioning, shapes, arrows.meta, calc}
\usepackage{bm}
\usepackage{dsfont}
\usepackage{mathrsfs}
\usepackage{subcaption}
\usepackage{authblk}

\usepackage{times}
\usepackage{microtype}

\newtheorem{theorem}{Theorem}
\newtheorem{lemma}{Lemma}
\newtheorem{proposition}{Proposition}
\newtheorem{remark}{Remark}

\begin{document}

\title{\textbf{Learning When to Automate: Queue Control in Human-AI Service Systems}}

\author[1,3]{Giovanni Montanari}
\author[4]{Marco Scarsini}
\author[1,2,3]{Vianney Perchet}

\affil[1]{FairPlay Joint Team, Inria, France}
\affil[2]{Criteo AI Lab, Paris, France}
\affil[3]{CREST, ENSAE, Institut Polytechnique de Paris}
\affil[4]{Department of Economics and Financial Markets, Luiss University}

\date{}
\maketitle

%
%
%
%
%
%

\begin{abstract}
We study a human-AI service system in which tasks arrive sequentially and are processed through a two-stage architecture: an automated chatbot followed, when necessary, by a human agent. We consider $T$ sequentially arriving tasks, each belonging to one of $K$ heterogeneous types. For each task the decision maker chooses how many resources to allocate to the chatbot, whose type-dependent success probabilities are initially unknown. Tasks not resolved by the chatbot enter type-dependent human-service queues, where they are processed by a human agent with unknown service rates. This model captures a central tradeoff in hybrid service systems: relying more on automation reduces human congestion but increases chatbot costs, while insufficient automation may overload the human agent. We propose the \texttt{UCB-DPP} policy, which combines Upper Confidence Bounds with Drift-Plus-Penalty control to learn the unknown parameters of the system while making queue-aware decisions. We prove that \texttt{UCB-DPP} achieves regret $\widetilde{\mathcal{O}}(K\sqrt{T})$ and guarantees mean-rate stability of the human-service queues. Simulations on synthetic instances show that the proposed policy outperforms natural baselines.
\end{abstract}

\section{Introduction}
The increasing deployment of large language models (LLMs) in service systems raises a fundamental question: how should automated agents and human operators be jointly coordinated? In many practical settings, an LLM-based assistant can resolve part of the incoming workload quickly, but its performance depends on the amount of computational or operational resources allocated to it and it may still fail on difficult requests. Human agents, on the other hand, are often more reliable but slower and capacity-constrained. A central challenge is therefore to exploit automation while controlling both operational costs and congestion in the human-service system.

We study this challenge through an online learning and queueing-control model for human-AI service systems. Tasks arrive sequentially over a finite horizon $T$ and belong to one of $K$ heterogeneous types, representing different levels of difficulty or service requirements. Each task is first processed by a chatbot. The decision maker chooses a cost level for the chatbot, interpreted as the amount of resources devoted to automated resolution. A higher cost increases the probability that the chatbot resolves the task, while tasks not resolved automatically are routed to a type-dependent human queue. The human agent can serve only one queue at a time, so the controller must also decide how to allocate limited human service capacity across task types.

The key feature of this setting is that automation and human scheduling control two different sides of the queueing system. The chatbot cost decisions shape the arrival process of the human queues: using more automation reduces future human workload but increases immediate chatbot costs. Scheduling decisions, instead, determine the departure process by allocating human service capacity across queues. Thus, automation and human service cannot be optimized independently: they must be coordinated over time in response to the current congestion state.

Another key difficulty is that the platform typically does not know in advance how different task types will be handled by the chatbot and by the human agent. Empirical evidence suggests that the relative effectiveness of AI and human service depends on task complexity \cite{xu2020ai}, while recent work on generative AI highlights that model performance may vary sharply across tasks that appear similar \cite{dellacqua2026navigating}. These observations motivate treating chatbot effectiveness and human service rates as task-dependent and initially unknown. The platform must therefore learn them from experience while simultaneously deciding how much to rely on automation and how to allocate limited human service capacity.

\subsection{Motivation}

The proposed framework is motivated by a broad class of human-AI service platforms in which automation and human intervention coexist. A canonical example is customer support, where incoming requests are first handled by a chatbot and escalated to a human operator only when the automated response is insufficient. This hybrid architecture is increasingly relevant in practice: chatbot-based customer service can generate economic value for firms~\cite{fotheringham2023effect}, but empirical studies also show that user adoption and satisfaction depend on factors such as communication quality, trust, privacy concerns, anthropomorphism, and the quality of the interaction in failure scenarios~\cite{sheehan2020customer,song2022will,hsu2023understanding,huang2024can,chen2023chatbot,cai2024communication,greilich2025consumer}. These findings suggest that replacing human service entirely with automation is often not desirable; instead, effective systems should decide when to rely on the chatbot and when to route unresolved requests to human agents.

The same two-stage structure appears beyond customer support. In corporate IT helpdesks, automated assistants may resolve routine access or software issues, while more complex tickets are routed to human technicians. In fraud detection, automated screening systems may clear routine transactions, whereas uncertain cases are forwarded to human analysts for manual review. In all these examples, the platform must jointly determine how much to rely on automation and how to allocate limited human service capacity across heterogeneous queues, while learning the effectiveness of both components online.

\subsection{Contributions}

The first contribution is a queueing model for human-AI service platforms in which automation decisions and human scheduling decisions are coupled. Chatbot decisions endogenously control the residual arrival process into the human queues, while scheduling decisions determine how the human agent serves the resulting queues. This two-stage structure captures a common workflow in modern service systems and differs from standard queueing-control models in which the arrival process is typically exogenous.

The second contribution is algorithmic. We propose the \texttt{UCB-DPP} policy, which combines Upper Confidence Bounds (UCB) with Drift-Plus-Penalty (DPP) control. The policy uses optimistic estimates of the unknown chatbot success probabilities and human service rates, while making queue-aware decisions through a weighted drift-plus-penalty objective. A key difficulty is that DPP is traditionally used to obtain asymptotic stability and time-average optimality guarantees, whereas here it must be adapted to a finite-horizon regret analysis. Moreover, learning errors affect both the arrival and the service sides of the queueing system.

The third contribution is theoretical. We prove that \texttt{UCB-DPP} achieves a sublinear regret bound of order $\widetilde{\mathcal O}(K\sqrt T)$ and ensures mean-rate stability of the human-service queues. The analysis combines UCB concentration bounds with Lyapunov-drift arguments and controls estimation errors weighted by the queue backlog. Finally, we complement the theoretical results with simulations on synthetic instances, showing that \texttt{UCB-DPP} outperforms natural baseline policies.

\subsection{Related Work}

\paragraph{Human-AI service systems and chatbot escalation.}
A growing literature studies service systems in which automated agents and human operators jointly handle customer requests. Hybrid customer-service architectures combining virtual agents and human operators have been proposed and analyzed from empirical and multidisciplinary perspectives~\cite{C-H_paper1,C-H_paper2}. Closer to the service-operations literature, human-AI service systems have been studied in settings with strategic customers, with a focus on whether the use of AI should be mandated~\cite{C-H_paper3}. Related work also analyzes human-AI interaction in congested service systems, but in settings where AI assists human decision-making rather than acting as a first-stage service channel~\cite{C-H_paper4}.

A related line of work studies the machine--human handoff problem in dialogue systems. This literature considers the task of predicting when a chatbot conversation should be transferred to a human agent~\cite{C-H_paper5}, jointly models handoff prediction and service satisfaction~\cite{C-H_paper6}, and incorporates cost-aware considerations into handoff decisions~\cite{C-H_paper7}. Other works study intent recognition mechanisms for identifying requests to switch from chatbot to human service~\cite{C-H_paper8}. Recent field evidence from large-scale customer-service operations further highlights the operational importance of human-in-the-loop interventions in agentic AI systems~\cite{C-H_paper9}.

This work takes a different perspective by combining online learning with queue control in a human-AI service model. The platform must jointly minimize chatbot costs and stabilize the human-service queues, while learning both the chatbot success probabilities and the human service rates across task types. We propose a learning policy for this setting and provide a finite-time regret analysis, showing how automation decisions and human-capacity allocation can be optimized jointly under uncertainty.

\paragraph{Queueing control and DPP.}
The present work is also related to the classical literature on queueing control and Lyapunov-based scheduling. A central line of work studies queue-length-based scheduling policies, such as MaxWeight and backpressure, which select service actions according to the current congestion state of the system. These policies originate from the seminal work on stability and maximum-throughput scheduling in constrained queueing networks \cite{queue_control1}, and have been further analyzed in generalized switch and heavy-traffic regimes \cite{queue_control2}. More broadly, queue-length-based scheduling and resource-allocation rules have become standard tools for the control of stochastic networks \cite{queue_control4,queue_control5}.

The algorithmic approach proposed in this work is particularly inspired by Lyapunov optimization and the Drift-Plus-Penalty methodology~\cite{Book_on_Queues,queue_control3}. This framework designs online control policies by minimizing, at each time step, an upper bound on the Lyapunov drift plus a weighted penalty term. In the analyzed setting, the drift term captures the evolution of the human-service queues, while the penalty corresponds to the chatbot cost. The main difference from the classical queueing-control literature is that the relevant system parameters, namely the chatbot success probabilities and the human service rates, are unknown and must be learned online while the queues are being controlled.

\paragraph{Online learning and bandits.}
The learning component of the model is related to the classical literature on multi-armed bandits and online learning. The bandit framework studies sequential decision-making problems in which a learner must balance exploration and exploitation while minimizing regret. Foundational results on adaptive allocation rules were established in~\cite{bandits1}, while finite-time regret guarantees for Upper Confidence Bound algorithms were developed in~\cite{bandits2}. We refer to~\cite{bandits3,bandits_book} for broader treatments of stochastic and adversarial bandit models.

Although \texttt{UCB-DPP} relies on UCB-type confidence bounds, the setting differs from standard bandit problems. 
Observations are generated by the control decisions and the queueing state: chatbot outcomes provide information only for task types on which resources are allocated, while human service outcomes are observed only when the policy selects a nonempty queue for service.
Moreover, estimation errors are weighted by the current backlog, coupling learning with queue control. Thus, standard bandit arguments must be combined with Lyapunov-drift techniques.

\paragraph{Online learning in queue control.}
Finally, the most closely related line of work studies online learning in queueing and stochastic network control systems, where control decisions must be made under unknown system parameters.
Several works combine bandit-learning ideas with queueing objectives, including the learning of unknown service rates \cite{learning_queue1}, scheduling with unknown statistics through MaxWeight and UCB-type rules \cite{learning_queue2}, and the analysis of transient congestion effects induced by learning \cite{learning_queue3}. Related approaches combine Lyapunov-drift methods with bandit feedback or adversarial learning in constrained scheduling and stochastic network optimization problems \cite{learning_queue4,learning_queue5,learning_queue6,DPOP}.

This work is related to the literature on learning in queueing systems, but the model has a different two-stage structure. In standard queueing-control problems, arrivals are typically exogenous and the controller acts mainly on the service side. In the proposed setting, chatbot decisions endogenously shape the arrival process into the human queues, while scheduling decisions determine the departure process. The platform must therefore learn and control both sides of the system: the chatbot layer, through unknown chatbot success probabilities, and the human-service layer, through unknown service rates.

\section{The Model}

\begin{figure}[H]
\centering
\resizebox{0.90\textwidth}{!}{
\begin{tikzpicture}[
    node distance=1.0cm and 1.0cm,
    >=Latex,
    font=\small,
    block/.style={draw, rounded corners, align=center, minimum width=2.0cm, minimum height=0.8cm, inner sep=3pt},
    smallblock/.style={draw, rounded corners, align=center, minimum width=1.7cm, minimum height=0.7cm, inner sep=3pt},
    line/.style={-Latex, thick}
]

\node[block] (arrival) {Task arrival\\ $X_t \in [K]$};

\node[block, right=of arrival] (chatbot) {Chatbot\\ cost $c_t$\\ success prob. $p_{X_t}c_t$};

\node[smallblock, above right=0.55cm and 0.8cm of chatbot] (success) {Resolved\\ automatically};

\node[block, below right=0.55cm and 0.8cm of chatbot] (queues) {Human queues\\ $(Q_{t,1},\dots,Q_{t,K})$};

\node[smallblock, right=of queues] (scheduler) {Scheduling\\ action $a_t$};

\node[block, right=of scheduler] (human) {Human agent\\ service prob. $\mu_{a_t}$};

\draw[line] (arrival) -- (chatbot);

\draw[line] (chatbot) -- node[above, sloped, font=\scriptsize] {success} (success);

\draw[line] (chatbot) -- node[above, sloped, font=\scriptsize] {failure} (queues);

\draw[line] (queues) -- (scheduler);

\draw[line] (scheduler) -- (human);

\end{tikzpicture}
}
\caption{Illustration of the human-chatbot service system.}
\label{fig:model}
\end{figure}
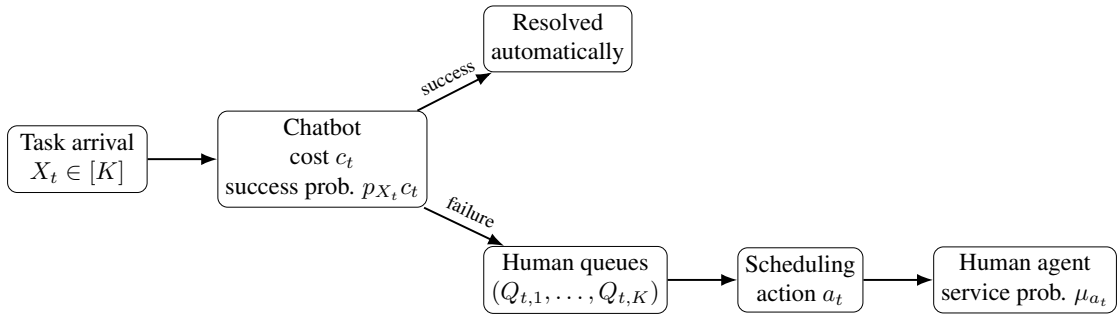

Tasks arrive sequentially over a finite time horizon $T$, which is assumed to be known in advance. At each time $t\in[T]$, a task of type $X_t$ enters the system, where
$$
X_t \in [K] := \{1,\ldots,K\}.
$$
Task types represent heterogeneous classes of requests, for instance different levels of difficulty, urgency, or complexity.

We assume that arrivals are independent and identically distributed over time according to an unknown distribution $\bm\lambda=(\lambda_1,\ldots,\lambda_K)$ over the set of task types. Thus, for every $k\in[K]$, $\lambda_k = \mathbb P(X_t=k)>0$, and $\sum_{k=1}^K \lambda_k=1$. The arrival distribution $\bm\lambda$ is not known to the decision maker.

\paragraph{Chatbot success probability.}
Upon arrival, each task is first routed to a chatbot. If the task arriving at time $t$ is of type $k$, the decision maker selects a cost level $c\in[0,1]$, representing the amount of resources allocated to the chatbot for handling that task. The chatbot then succeeds in resolving the task with probability $s_k(c)$, and fails otherwise. Equivalently, the chatbot outcome is a Bernoulli random variable with parameter $s_k(c)$.

We assume that the chatbot success probability is linear in the allocated cost. Namely, for each type $k\in[K]$, there exists an unknown parameter $p_k\in(0,1]$ such that
\begin{equation}\label{eq: form of the success proba}
    s_k(c)=p_k c, \qquad c\in[0,1].
\end{equation}
Thus, $p_k$ captures the intrinsic effectiveness of the chatbot on tasks of type $k$, while the control variable $c$ modulates the actual probability of automatic resolution.

\paragraph{Human service mechanism.}
Tasks that are not resolved by the chatbot are routed to a human service system. We assume that there exist unknown parameters $\{\mu_1,\ldots,\mu_K\}$ such that, whenever the human agent works on a task of type $k$, the task is completed by the end of the slot with probability $\mu_k\in(0,1]$. Equivalently, each slot devoted to a type-$k$ task produces an independent Bernoulli service attempt with success probability $\mu_k$. Hence, conditional on the task type $k$, the service time $L$ (measured in number of slots) follows a geometric distribution and satisfies
\begin{equation}
    \mathbb{E}[L \mid k] = \frac{1}{\mu_k}.
\end{equation}
To capture heterogeneity across task types, we consider $K$ separate queues, one for each type, and introduce a scheduling decision for the human server. At each slot $t$, the decision maker chooses an action
\begin{equation}
    a_t \in \{1,\ldots,K\}.
\end{equation}
If $a_t=k$ and the corresponding queue is nonempty, then the human agent serves one type-$k$ task during slot $t$. We assume that scheduling is \emph{preemptive}: the server may switch across task types from one slot to the next, interrupting the processing of a task and possibly returning to it later. This assumption is natural in the geometric service model, where each service attempt is memoryless and the probability of completion depends only on the type currently being served.

\paragraph{Queue dynamics.}
Let $\bm Q_t := (Q_{t,1},\ldots,Q_{t,K})$ denote the backlog vector at the beginning of slot $t$, where $Q_{t,k}\in\mathbb N$ is the number of pending type-$k$ tasks in the human queue. We assume that all queues are initially empty, namely $Q_{1,k}=0$ for all $k\in[K]$.

At each slot $t$, one task of type $X_t\in[K]$ arrives and is first routed to the chatbot. After the decision maker selects the chatbot cost $c_t\in[0,1]$, the task either leaves the system if the chatbot succeeds, or joins the human queue corresponding to its type if the chatbot fails. We define the per-type arrival indicators to the human queues as
\begin{equation}
    A_{t,k}
    :=
    \mathds{1}\{X_t=k\}\,
    \mathds{1}\{\text{the chatbot fails at slot }t\},
    \qquad k\in[K].
\end{equation}
Thus, at most one human queue receives an arrival in each slot. Moreover, conditional on $X_t=k$ and $c_t$, the expected arrival to queue $k$ is
\begin{equation}
    \mathbb E[A_{t,k}\mid X_t=k,c_t]=1-p_kc_t,
\end{equation}
while $A_{t,j}=0$ for all $j\neq k$.

The scheduling action $a_t\in[K]$ determines which queue is selected for service. If $a_t=k$ and queue $k$ is nonempty, the human agent attempts to complete one type-$k$ task during slot $t$. We define the departure indicators as
\begin{equation}
    D_{t,k}
    :=
    \mathds{1}\{a_t=k\}\,
    \mathds{1}\{Q_{t,k}>0\}\,
    Z_{t,k},
    \qquad k\in[K],
\end{equation}
where $Z_{t,k}\sim \mathcal{B}(\mu_k)$. Hence, conditional on the current backlog and on the scheduling decision,
\begin{equation}
    \mathbb E[D_{t,k}\mid \bm Q_t,a_t]
    =
    \mathds{1}\{a_t=k\}\mathds{1}\{Q_{t,k}>0\}\mu_k.
\end{equation}
In particular, $D_{t,k}=0$ whenever $a_t\neq k$ or $Q_{t,k}=0$.

The human queues then evolve according to
\begin{equation}\label{eq: queue dynamics}
    Q_{t+1,k}=Q_{t,k}+A_{t,k}-D_{t,k},
    \qquad k\in[K].
\end{equation}
Since $D_{t,k}\le \mathds{1}\{Q_{t,k}>0\}$, the queue lengths remain nonnegative at all times.

\begin{remark}
    Within each slot $t$, the decision maker first observes $(\mathbf{Q}_t,X_t)$ and chooses $(c_t,a_t)$; then the chatbot outcome and the service outcome are realized, yielding $A_{t,k}$ and $D_{t,k}$, and finally the queues are updated according to Equation \eqref{eq: queue dynamics}.
    Tasks forwarded to the human system during slot $t$ become available for service only from slot $t+1$ onward.
\end{remark}

Following \cite{Book_on_Queues}, we say that queue $k$ is \emph{mean rate stable} if its expected backlog grows sublinearly over time, namely if
\begin{equation}\label{eq: stability definition}
    \lim_{T\to\infty}\frac{\mathbb E[Q_{T,k}]}{T}=0.
\end{equation}
Equivalently, the queue does not accumulate unresolved tasks at a linear rate.

Figure \ref{fig:model} provides a schematic illustration of the model.

\subsection{Cost Objective and Regret}
A control policy specifies both the chatbot-cost decision and the human-service scheduling action over time. Formally, an admissible policy $\pi$ is a decision rule that, at each slot $t$, maps the current task type, the current backlog vector, and the past history into an action pair
\begin{equation}\label{eq: policy general form}
    \pi_t:\ (X_t,\bm Q_t^\pi,\mathcal F_{t-1})\mapsto (c_t^\pi,a_t^\pi),
\end{equation}
where $c_t^\pi\in[0,1]$ is the cost allocated to the chatbot and $a_t^\pi\in[K]$ is the queue selected for human service. Here $\mathcal F_{t-1}$ denotes the information available up to the end of slot $t-1$.

Given a horizon $T$, we evaluate a policy through the total cost
\begin{equation}\label{eq: cost definition}
    C^\pi_T := \sum_{t=1}^T c_t^\pi + \sum_{k=1}^K r_k\,Q_{T+1,k}^\pi,
\end{equation}
where $r_k\in[r_{\min},r_{\max}]$, with $0<r_{\min}<r_{\max}$, is the penalty weight associated with unfinished tasks of type $k$. The first term is the cumulative chatbot cost, while the second term penalizes the residual backlog at the end of the horizon. Thus, larger values of $r_k$ correspond to task classes for which leaving jobs unresolved is more costly.

Since the system parameters are unknown, the decision maker must learn them while controlling the system. We measure the performance of a policy through regret with respect to the best admissible policy that knows the true parameters in advance. Let $\Pi$ denote the class of admissible policies satisfying \eqref{eq: policy general form}. The regret of a policy $\pi\in\Pi$ is defined as
\begin{equation}\label{eq: regret definition}
    R_T^\pi := \mathbb E[C^\pi_T] - \inf_{\pi'\in\Pi}\mathbb E[C^{\pi'}_T].
\end{equation}
The expectation is taken with respect to the randomness of arrivals, chatbot outcomes, service completions, and any possible randomization of the policy. Regret therefore measures the excess expected cost incurred by $\pi$ relative to the best parameter-aware admissible policy over the same horizon.

\subsection{Lower Bound via Static Optimization Problem}
Characterizing the optimal policy in \eqref{eq: regret definition} is generally difficult, since such a policy may depend both on the true system parameters $\bm\theta = (\bm \lambda,\bm p,\bm \mu)$ and on the current backlog state $\bm Q_t$. To obtain a tractable benchmark, we introduce the following \emph{static} optimization problem:
\begin{equation}\label{eq: static opt problem}
\begin{aligned}
\text{OPT}(\bm\theta) = \min_{\substack{c_k \in [0,1] \\ \rho_k \ge 0}}
\quad & \sum_{k=1}^K \lambda_k c_k \\
\text{s.t.} \quad
&\lambda_k(1-p_k c_k)\le \mu_k \rho_k,\qquad \forall k\in[K], \\
& \sum_{k=1}^K \rho_k \le 1.
\end{aligned}
\end{equation}
Here, $c_k$ represents the chatbot cost that would be assigned to all tasks of type $k$ under a static type-dependent rule. The quantity $\lambda_k(1-p_kc_k)$ is the effective rate at which type-$k$ tasks are not resolved by the chatbot and therefore enter the human queue. The variable $\rho_k$ represents the fraction of human service capacity allocated to type $k$, so that $\mu_k\rho_k$ is the corresponding effective service capacity. The constraints impose that, for each type, the human service capacity is large enough to absorb the residual arrival rate, while the total allocated human capacity cannot exceed one.

The problem is static because the decision variables $(\bm c,\bm\rho)$ are fixed over time and do not depend on the instantaneous backlog of the system. It therefore provides a parameter-dependent benchmark, rather than a full dynamic policy.

We assume that the optimization problem \eqref{eq: static opt problem} is feasible. Equivalently, since the residual load is minimized by taking $c_k=1$ for all $k$, feasibility is guaranteed if and only if
$$
\sum_{k=1}^K \frac{\lambda_k(1-p_k)}{\mu_k}\le 1.
$$

The Lagrangian associated with the optimization problem \eqref{eq: static opt problem} is
\begin{equation}
\mathcal{L}(\bm c,\bm \rho;\bm y,\nu) = \sum_{k=1}^K \lambda_k c_k + \sum_{k=1}^K y_k\left(\lambda_k(1-p_kc_k)-\mu_k\rho_k\right) + \nu\left(\sum_{k=1}^K \rho_k-1\right),
\end{equation}
where $y_k\ge 0$ is the dual multiplier associated with the $k$-th balance constraint, and $\nu\ge 0$ is the dual multiplier associated with the total capacity constraint. Let $(\bm y^*,\nu^*)$ be an optimal dual solution.

We can now state a lower bound linking the static benchmark to the cost of the optimal admissible policy.

\begin{theorem}\label{thm: static lower bound}
Assume that the terminal backlog weights satisfy
$$
r_k \ge y_k^*, \qquad \forall k\in[K].
$$
Then, for every horizon $T\ge 1$,
$$
\inf_{\pi\in\Pi}\mathbb E[C^\pi_T] \geq T \cdot \text{OPT}(\bm \theta).
$$
\end{theorem}

The multiplier $y_k^*$ can be interpreted as the shadow congestion cost of an additional unit of residual type-$k$ load. The condition $r_k\ge y_k^*$ therefore requires the terminal penalty to be large enough to internalize the congestion cost captured by the static benchmark. In particular, it rules out settings in which residual backlog is underpriced relative to the marginal value of service capacity.

Let $(\bm{c}^*, \bm{\rho}^*)$ denote an optimal solution to the static optimization problem \eqref{eq: static opt problem}. The following proposition is a direct consequence of Theorem~\ref{thm: static lower bound}.

\begin{proposition}\label{propos: first regret bound}
    For any policy $\pi$, the regret $R_T^\pi$ satisfies
    $$
    R_T^\pi \leq \sum_{t=1}^T\mathbb{E}[c_t^\pi-c_{X_t}^*]+ \sum_{k=1}^K r_k \, \mathbb{E}[Q_{T+1,k}^\pi].
    $$
\end{proposition}
The proofs of this section are in Appendix \ref{app:proof_static_lower_bound}.

\section{\texttt{UCB-DPP} Policy}
The goal is to design an online policy that jointly chooses, at each slot $t$, the chatbot cost $c_t\in[0,1]$ and the human-service scheduling action $a_t\in[K]$. The policy uses the current backlog vector $\bm Q_t$ together with the information collected about the unknown chatbot success probabilities $\bm p=(p_1,\ldots,p_K)$ and human service rates $\bm\mu=(\mu_1,\ldots,\mu_K)$. Its objective is to minimize cumulative chatbot costs while keeping the human queues stable in the sense of Equation~\eqref{eq: stability definition}. Throughout this section, we omit the superscript $\pi$ from all policy-dependent quantities.

\paragraph{Estimation of $\bm\mu$.}
Whenever the human agent serves a nonempty queue of type $k$, the service attempt succeeds with probability $\mu_k$. Thus, each effective service attempt on type $k$ produces a Bernoulli observation. For each $k\in[K]$, define
\begin{equation}\label{eq: service counters}
    N_k(t) := \sum_{s=1}^t \mathds{1}\{a_s = k,\; Q_{s,k}>0\},
    \qquad
    S_k(t) := \sum_{s=1}^t D_{s,k}.
\end{equation}
Here, $N_k(t)$ is the number of effective service attempts on type $k$ up to time $t$, while $S_k(t)$ is the number of successful completions of type-$k$ tasks. We set $N_k(0)=S_k(0)=0$. The empirical estimator of $\mu_k$ is
\begin{equation}\label{eq: estimator of mu}
\hat{\mu}_k(t) :=
\begin{cases}
0, & \text{if } N_k(t)=0,\\[1mm]
\dfrac{S_k(t)}{N_k(t)}, & \text{if } N_k(t)\ge 1.
\end{cases}
\end{equation}
Given a confidence radius $\beta_k^\mu(t)\ge 0$, defined below, we use the optimistic estimate
\begin{equation}\label{eq: mu_ucb}
    \bar{\mu}_k(t) := \min\left\{1,\hat{\mu}_k(t)+\beta_k^\mu(t)\right\}.
\end{equation}

\paragraph{Estimation of $\bm p$.}
When a task of type $k$ arrives at slot $t$ and the chatbot cost $c_t$ is chosen, the chatbot succeeds with probability $p_kc_t$. Let $Y_t := \mathbf{1}\{\text{chatbot succeeds at slot } t\}$. Conditionally on $(X_t=k,c_t)$,
$$
\mathbb{E}[Y_t \mid X_t=k,c_t] = p_k c_t.
$$
To estimate $p_k$, for each $k\in[K]$ we define
\begin{equation}\label{eq: chatbot counters}
    M_k(t) := \sum_{s=1}^t c_s \mathds{1}\{X_s=k\},
    \qquad
    G_k(t) := \sum_{s=1}^t Y_s \mathds{1}\{X_s=k\}.
\end{equation}
Here, $M_k(t)$ is the cumulative exposure of type $k$ to the chatbot, while $G_k(t)$ is the cumulative number of chatbot successes on type $k$. As shown below, the \texttt{UCB-DPP} cost decision is threshold-based and therefore satisfies $c_t\in\{0,1\}$. Hence $M_k(t)$ counts the number of effective chatbot trials on type $k$. We set $M_k(0)=G_k(0)=0$.
The empirical estimator of $p_k$ is
\begin{equation}\label{eq: estimator of p}
\hat p_k(t):=
\begin{cases}
0, & \text{if } M_k(t)=0,\\[1mm]
\dfrac{G_k(t)}{M_k(t)}, & \text{if } M_k(t)>0.
\end{cases}
\end{equation}
Given a confidence radius $\beta_k^p(t)\ge 0$, we define
\begin{equation}\label{eq: p_ucb}
    \bar p_k(t):=\min\left\{1,\hat p_k(t)+\beta_k^p(t)\right\}.
\end{equation}

\paragraph{Drift-Plus-Penalty principle.}
We use a weighted quadratic Lyapunov function,
\begin{equation}\label{eq: lyapunov function}
    \Phi(\bm{Q}_t)=\frac12\sum_{k=1}^K r_k Q_{t,k}^2.
\end{equation}
Let us define the conditional drift as
\begin{equation}\label{eq: drift}
    \Delta_t = \mathbb{E}\big[\Phi(\bm{Q}_{t+1})-\Phi(\bm{Q}_t)\mid \mathcal{F}_t\big].
\end{equation}

The following drift bound follows directly from the queue dynamics.
\begin{lemma}[Drift Upper Bound] \label{lemma: bound drift DPP preemptive}
For all $t \in [T]$, it holds that
$$
\Delta_t \le r_{\max} + \sum_{k=1}^K r_k Q_{t,k}\,\mathbb{E}[A_{t,k}-D_{t,k}\mid \mathcal{F}_t].
$$
\end{lemma}
The proof is in Appendix \ref{app: UCB-DPP policy}.

Following the Drift-Plus-Penalty methodology, the control actions are chosen by minimizing an upper bound on
\begin{equation}\label{eq: original DPP minimization}
    \Delta_t + V c_t,
\end{equation}
where the parameter $V>0$ balances instantaneous chatbot costs against queue growth. Larger values of $V$ put more weight on cost minimization, while smaller values prioritize congestion control.

Using Lemma \ref{lemma: bound drift DPP preemptive}, we have
\begin{equation}\label{eq: objective DPP preemptive setting}
    \Delta_t+Vc_t \leq r_{\max}+\sum_{k=1}^K r_k Q_{t,k}\,\mathbb{E}[A_{t,k}-D_{t,k}\mid \mathcal{F}_t]+Vc_t.
\end{equation}
Substituting the conditional expectations of the arrival and departure processes, and dropping the constant term $r_{\max}$, yields the surrogate objective
\begin{equation}\label{eq: surrogate objective}
    \Psi_t(c,a):=Vc + r_{X_t} Q_{t,X_t}(1-p_{X_t}c) - r_a Q_{t,a}\mu_a.
\end{equation}

Since the parameters $\bm p$ and $\bm\mu$ are unknown, the \texttt{UCB-DPP} policy replaces them with their optimistic estimates and minimizes
\begin{equation}\label{eq: surrogate objective both unknown}
    \widetilde{\Psi}_t(c,a):=Vc + r_{X_t} Q_{t,X_t}(1-\bar p_{X_t}(t-1)c) - r_a Q_{t,a}\bar{\mu}_a(t-1).
\end{equation}
Thus, at each slot $t$,
\begin{equation}\label{eq: argmin surrogate both unknown}
(c_t,a_t)\in \arg\min_{c\in[0,1],\, a\in[K]} \widetilde{\Psi}_t(c,a).
\end{equation}

Since $\widetilde\Psi_t$ is separable in $c$ and $a$ and linear in $c$, the policy admits the explicit form
$$
c_t=
\begin{cases}
1, & r_{X_t}\bar p_{X_t}(t-1)Q_{t,X_t}\ge V,\\
0, & \text{otherwise,}
\end{cases}
\qquad
a_t\in\arg\max_{j\in[K]} r_jQ_{t,j}\bar\mu_j(t-1).
$$
If all queues are empty, the scheduling decision is irrelevant. 

Thus, the chatbot is activated when its weighted optimistic benefit in reducing future human backlog exceeds the cost threshold $V$, while the human server is assigned to the queue with the largest optimistic weighted service potential, as in a MaxWeight rule.

The threshold structure of the chatbot decision follows from the linear form $s_k(c)=p_kc$: since the surrogate objective is linear in $c$, the chatbot cost always takes values in $\{0,1\}$, corresponding to an \emph{on/off} decision. Learning enters the policy through two UCB mechanisms \cite{bandits2}: one for the chatbot success probabilities and one for the human service rates. The former affects the arrival process into the human queues, while the latter affects the scheduling of human service capacity.

A complete pseudocode description of the \texttt{UCB-DPP} policy is provided in Appendix \ref{app: UCB-DPP policy}.

\section{Regret Bound}
We now show that the regret of the \texttt{UCB-DPP} policy grows sublinearly with the horizon $T$. As a consequence of the backlog estimates used in the analysis, we also obtain mean-rate stability of the human-service queues.

The first step is to decompose the regret into terms that can be controlled separately. For every $t\in[T]$, define the \emph{good event} up to time $t$ as
\begin{equation}\label{eq: good event}
\mathscr{G}_t = \bigcap_{s=1}^t \bigcap_{k=1}^K \left( \left\{|\mu_k-\hat{\mu}_k(s-1)|\le \beta_k^\mu(s-1) \right\} \cap \left\{ |p_k-\hat p_k(s-1)|\le \beta_k^p(s-1) \right\} \right).
\end{equation}
On this event, all empirical estimates of the chatbot success probabilities and human service rates remain within their confidence intervals uniformly over all types and all times up to $t$.

\begin{lemma}[Regret decomposition]\label{lemma: regret decomposition}
The regret of the \texttt{UCB-DPP} policy satisfies
$$
R_T^{\texttt{UCB-DPP}} \leq R_1^p(T)+R_1^\mu(T)+R_2(T)+R_3(T)+R_4(T),
$$
where
\begin{itemize}
    \item $R_1^p(T) := \frac{1}{V} \mathbb E\!\left[ \sum_{t=1}^T r_{X_t}Q_{t,X_t}c_t\left(\bar p_{X_t}(t-1)-p_{X_t}\right)\mathds{1}\{\mathscr G_t\} \right]$,
    \item $R_1^\mu(T):=\frac{1}{V}\mathbb E\!\left[\sum_{t=1}^Tr_{a_t}Q_{t,a_t}\left(\bar\mu_{a_t}(t-1)-\mu_{a_t}\right)\mathds{1}\{\mathscr G_t\}\right]$,
    \item $R_2(T) := T\,\mathbb P(\overline{\mathscr G}_T)$,
    \item $R_3(T):=\frac{r_{\max}T}{V}$,
    \item $R_4(T):=\sum_{k=1}^K r_k\,\mathbb E[Q_{T+1,k}]$.
\end{itemize}
\end{lemma}

We now bound the terms appearing in Lemma~\ref{lemma: regret decomposition}.

\paragraph{Control of the bad event.}
We first control the probability of the complement of the good event. Fix $\delta\in(0,1)$, whose value will be specified later. We set $\beta_k^\mu(0):=1$, $\beta_k^p(0):=1$ and, for every $t\ge 1$, define

\begin{equation}\label{eq: confidence radii} 
\begin{aligned} 
\beta_k^\mu(t) &= \begin{cases} 1, & \text{if } N_k(t)=0,\\[-1mm] \sqrt{\dfrac{\log(4Kt/\delta)}{2N_k(t)}}, & \text{if } N_k(t)\ge 1, 
\end{cases} \qquad 
\beta_k^p(t) &= \begin{cases} 1, & \text{if } M_k(t)=0,\\[-1mm] \sqrt{\dfrac{\log(4Kt/\delta)}{2M_k(t)}}, & \text{if } M_k(t)\ge 1. 
\end{cases} 
\end{aligned} 
\end{equation}

\begin{proposition}\label{propos: bound on bad event proba}
The event $\mathscr{G}_T$ satisfies
$$
\mathbb{P}(\overline{\mathscr{G}}_T) \leq \delta T.
$$
\end{proposition}

\paragraph{Control of the terminal backlog.}
We next control the backlog contribution $R_4(T)$. Define
\begin{equation}\label{eq: definition Z_t}
    Z_t:=\sqrt{2\Phi(\bm Q_t)} = \left(\sum_{k=1}^K r_kQ_{t,k}^2\right)^{1/2}.
\end{equation}

\begin{lemma}\label{lemma:deterministic_weighted_backlog}
Assume that there exist $\varepsilon>0$ and a strictly feasible point $(\bm c^\varepsilon,\bm \rho^\varepsilon)$ such that
$$
\lambda_k(1-p_kc_k^\varepsilon)+\varepsilon\le \mu_k\rho_k^\varepsilon, \; \forall k\in[K] \;\;\; \text{and} \;\;\; \sum_{k=1}^K \rho_k^\varepsilon\le 1.
$$
Then 
$$ 
\mathbb E[Z_{T+1}\mathds{1}\{\mathscr G_T\}] \leq \mathcal{O}(V + K \log (T/\delta)).
$$
\end{lemma}

The assumption in Lemma \ref{lemma:deterministic_weighted_backlog} is a strong stability condition for the static benchmark. It requires the existence of a type-dependent chatbot allocation and a human-capacity allocation such that, for every class $k$, the effective service capacity $\mu_k\rho_k^\varepsilon$ exceeds the residual arrival rate $\lambda_k(1-p_kc_k^\varepsilon)$ by a uniform margin $\varepsilon$. This slack is used in the drift argument to obtain a negative drift when the backlog is large. Intuitively, it rules out boundary cases in which the system is only critically loaded, where queues may be stable asymptotically but finite-horizon backlog bounds are harder to control.
As a consequence, we obtain the following bound on the terminal backlog term.

\begin{proposition}[Bound on $R_4(T)$]\label{prop:R4_bound_weighted}
Under the assumptions of Lemma~\ref{lemma:deterministic_weighted_backlog}, $R_4(T)$ satisfies
$$
R_4(T) = \mathcal O(V+K\log (T/\delta)) + r_{\max}T\,\mathbb P(\overline{\mathscr G}_T).
$$
\end{proposition}

\paragraph{Control of the estimation-error terms.}
It remains to control the two terms involving estimation errors on the good event, namely $R_1^\mu(T)$ and $R_1^p(T)$.

\begin{proposition}\label{prop:R1_weighted_bounds}
The estimation-error terms satisfy
$$
R_1^\mu(T),\; R_1^p(T) = \mathcal{O}\left(K \sqrt{T\log(T/\delta)}\right).
$$
\end{proposition}

\paragraph{Final regret bound.}
We are now ready to combine the preceding bounds and state the main regret guarantee.
\begin{theorem}\label{thm: main regret bound}
Under the assumption in Lemma~\ref{lemma:deterministic_weighted_backlog}, by setting $V= \sqrt{T}$ and $\delta=T^{-2}$, the \texttt{UCB-DPP} policy satisfies, for $T\ge 2$,
$$
R_T^{\texttt{UCB-DPP}} = \mathcal{O}\left(K\sqrt{T\log T}\right).
$$
\end{theorem}
The proof starts from the regret decomposition in Lemma \ref{lemma: regret decomposition}. On the good event, optimism reduces the regret to two estimation-error terms, corresponding to the chatbot and service parameters. These terms are nonstandard because the estimation errors are weighted by the current queue backlog. They are controlled by combining concentration bounds with Lyapunov-drift arguments that keep the relevant backlog terms under control. The remaining terms are the bad-event probability, the drift contribution $T/V$, and the terminal backlog, which is of order $V$ up to logarithmic factors. Setting $V=\sqrt T$ and $\delta=T^{-2}$ balances these terms and gives the stated regret bound.

We can also observe that, as a direct consequence of Proposition \ref{prop:R4_bound_weighted}, choosing $V=\sqrt T$ and $\delta=T^{-2}$ gives a sublinear weighted terminal backlog. Since $r_k\ge r_{\min}>0$ for all $k\in[K]$, each human-service queue is \emph{mean-rate stable}:
$$
\lim_{T\to\infty}\frac{\mathbb E[Q_{T+1,k}]}{T}=0,\qquad \forall k\in[K].
$$

The complete proofs of this section are provided in Appendix \ref{app: all regret bound}.

\section{Simulations}
We test the \texttt{UCB-DPP} policy on a synthetic instance with $K=5$ task classes, arrival rates $\bm \lambda = (0.32,0.08,0.25,0.20,0.15)$, chatbot success probabilities $\bm p = (0.82,0.35,0.68,0.25,0.9)$, human service rates $\bm \mu = (0.3,0.85,0.45,0.75,0.28)$, and terminal backlog weights $\bm r=(2.3,0.8,1.5,0.9,2.5)$, chosen to satisfy the assumption of Lemma \ref{lemma:deterministic_weighted_backlog}.

All simulations are run over a horizon $T=30000$ for the \texttt{UCB-DPP} plot in Figure \ref{fig:regret UCB-DPP} and over a horizon $T=10000$ for the policy comparison plot in Figure \ref{fig: comparison plot}. Each curve is averaged over $100$ independent runs, and the shaded regions represent one standard error.

As shown in Figure \ref{fig:regret UCB-DPP}, the cumulative regret of \texttt{UCB-DPP} grows sublinearly over time.

For the comparison plot, we consider three benchmark policies. The first is a plug-in version of DPP, which has the same decision structure as \texttt{UCB-DPP} but replaces the optimistic estimates $\bar p_k(t)$ and $\bar\mu_k(t)$ with the empirical estimates $\hat p_k(t)$ and $\hat\mu_k(t)$.
The other two policies always set the chatbot cost to either $c=1$ or $c=0$, respectively, and use a greedy scheduling rule based only on the empirical service rates:
$$
a_t \in \arg\max_{k\in[K]} r_k \hat\mu_k(t-1).
$$
This rule does not take the current backlog into account. Instead, at each slot it prioritizes the task type with the largest product between terminal backlog cost and estimated service rate, namely the class that is both more costly to leave unresolved and empirically easier to serve.

The comparison in Figure \ref{fig: comparison plot} therefore highlights the role of both the optimism mechanism and the queue-aware Drift-Plus-Penalty structure in stabilizing the system while controlling chatbot usage. Other simulations are provided in Appendix \ref{app: simulations}.

\begin{figure}[H] 
\centering 
\begin{subfigure}[t]{0.49\textwidth} 
\centering 
\vspace{0pt} 
\includegraphics[width=\textwidth]{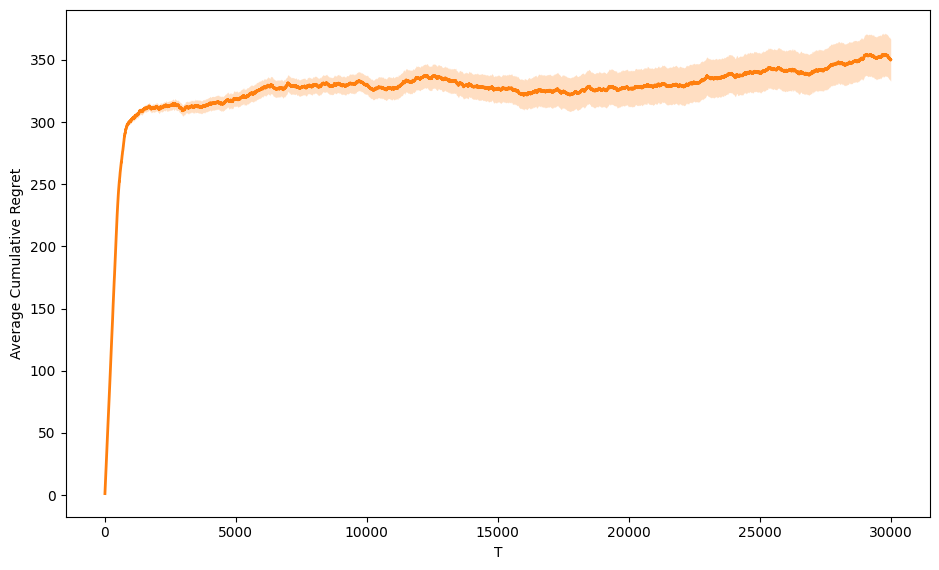} 
\caption{Average cumulative regret of \texttt{UCB-DPP}.} 
\label{fig:regret UCB-DPP} 
\end{subfigure} \hfill 
\begin{subfigure}[t]{0.49\textwidth} 
\centering 
\vspace{0pt} 
\includegraphics[width=\textwidth]{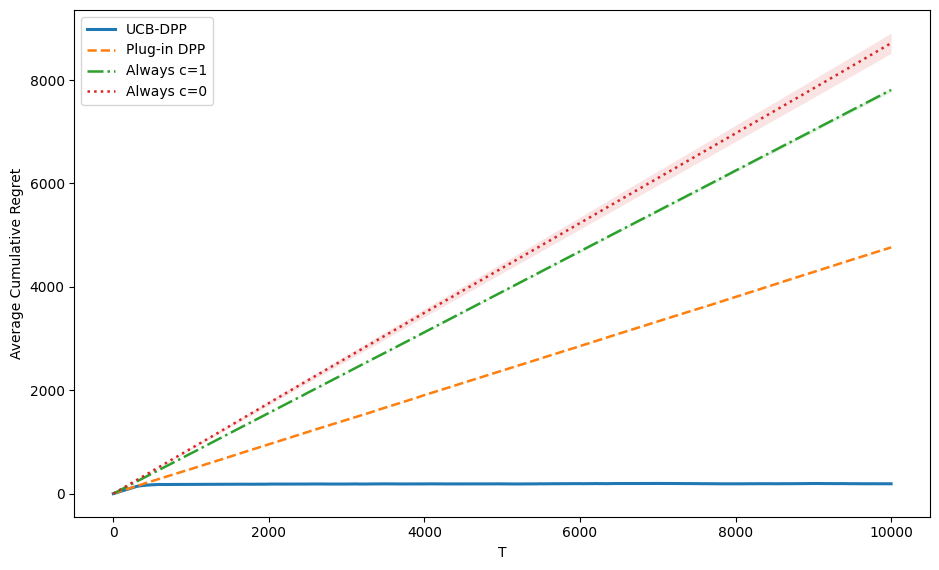} 
\caption{Average cumulative regret of \texttt{UCB-DPP} and three baselines.} 
\label{fig: comparison plot} 
\end{subfigure} 
\label{fig:simulations-paper} 
\end{figure}

\section{Conclusion, Limitations and Open Directions}

We studied an online learning and queueing-control model for human-AI service systems. The main insight is that automation and human scheduling cannot be optimized separately: chatbot decisions shape the arrival process into the human queues, while scheduling decisions determine how limited human capacity is allocated. To address this problem, we proposed the \texttt{UCB-DPP} policy, which combines optimistic parameter estimation with a Drift-Plus-Penalty objective. We proved that \texttt{UCB-DPP} achieves regret $\widetilde{\mathcal O}(K\sqrt T)$ and ensures mean-rate stability of the human-service queues. Simulations on synthetic instances show that the policy outperforms natural baselines.

The model relies on a few simplifying assumptions that make the analysis tractable and could be extended in future work. The linear chatbot success probability leads to an \emph{on/off} chatbot decision, with $c_t\in\{0,1\}$, so that the chatbot is either fully activated or not used. More general response functions could capture richer chatbot behavior and allow for intermediate levels of automation. Similarly, the geometric and preemptive human-service model could be generalized to non-memoryless service times and non-preemptive scheduling. Finally, another natural direction is to consider systems with multiple human agents, possibly with heterogeneous skills, different service rates, or strategic behavior. This would lead to richer allocation and incentive-design questions, especially in settings where human agents may respond strategically to the way tasks are routed or prioritized.

\section*{Acknowledgments}
This work was supported in part by the French National Research Agency (ANR) through grants ANR-20-CE23-0007 and ANR-23-CE23-0002, and through the PEPR IA FOUNDRY project (ANR23-PEIA-0003). It was also supported by Hi! PARIS and the ANR/France 2030 program (ANR23-IACL-0005).

\nocite{*}
\bibliographystyle{plain}
\bibliography{bibliography}

\appendix
\setcounter{section}{0}
\renewcommand{\thesection}{\Alph{section}}
\renewcommand{\theHsection}{appendix.\Alph{section}}

\section{Lower Bound via Static Optimization Problem} 
\label{app:proof_static_lower_bound}

\subsection{Proof of Theorem \ref{thm: static lower bound}}

For all $k \in [K]$, let us sum over $t$ the queue dynamics Equation \eqref{eq: queue dynamics} and use $Q_{1,k}^\pi=0$:
$$
Q_{T+1,k}^\pi = \sum_{t=1}^T(A_{t,k}^\pi-D_{t,k}^\pi).
$$
For simplicity of notation, let $\sum_t$ denote $\sum_{t=1}^T$. Taking expectation and dividing by $T$, we obtain
\begin{equation}\label{eq: average arr, dep, back}
\frac{1}{T}\mathbb{E}\left[\sum_t A_{t,k}^\pi \right] = \frac{1}{T}\mathbb{E}\left[\sum_t D_{t,k}^\pi \right] + \frac{1}{T}\mathbb{E}\left[ Q_{T+1,k}^\pi \right].
\end{equation}

Now let us rewrite the objects associated with a policy $\pi$ in a static form.
Define the average cost per type as
$$
\bar{c}_k^\pi = \frac{1}{T\lambda_k} \mathbb{E}\left[\sum_t c_t^\pi \mathds{1}\{X_t=k\} \right],
$$
then it follows that
$$
\frac{1}{T}\mathbb{E}\left[\sum_t c_t^\pi \right] = \sum_k \lambda_k \bar{c}_k^\pi.
$$

Let
$$
U_{t,k}^\pi := \mathds{1}\{a_t^\pi=k,\;Q_{t,k}^\pi>0\}
$$
be the indicator that, at slot $t$, the human agent serves a non-empty type-$k$ queue. Define
$$
\bar{\rho}_k^\pi = \frac{1}{T}\mathbb{E}\left[\sum_t U_{t,k}^\pi \right].
$$
Since at each slot the agent works on at most one type, it follows that
$$
\sum_{k=1}^K \bar{\rho}_k^\pi \leq 1.
$$

Finally define the average final backlog per unit of time as
\begin{equation}\label{eq: stat LB proof, backlog}
    \bar{q}_k^\pi = \frac{1}{T}\mathbb{E}\left[Q_{T+1,k}^\pi \right].
\end{equation}
Recall that for the arrival process, given $X_t = k$,
$$
\mathbb{E}[A_{t,k}^\pi \mid X_t=k,c_t^\pi] = 1-p_kc_t^\pi.
$$
Therefore, by the tower property,
$$
\mathbb{E}[A_{t,k}^\pi]
=
\mathbb{E}\!\left[\mathbb{E}[A_{t,k}^\pi\mid X_t,c_t^\pi]\right].
$$
Since $A_{t,k}^\pi=0$ whenever $X_t\neq k$, we get
$$
\mathbb{E}[A_{t,k}^\pi]
=
\mathbb{E}\!\left[\mathbf{1}\{X_t=k\}\,\mathbb{E}[A_{t,k}^\pi\mid X_t=k,c_t^\pi]\right]
=
\mathbb{E}\!\left[\mathbf{1}\{X_t=k\}(1-p_kc_t^\pi)\right].
$$
Expanding the right-hand side,
$$
\mathbb{E}[A_{t,k}^\pi]
=
\mathbb{E}[\mathbf{1}\{X_t=k\}]
-
p_k\,\mathbb{E}\!\left[\mathbf{1}\{X_t=k\}c_t^\pi\right].
$$
Using $\mathbb{P}(X_t=k)=\lambda_k$, this becomes
$$
\mathbb{E}[A_{t,k}^\pi]
=
\lambda_k
-
p_k\,\mathbb{E}\!\left[\mathbf{1}\{X_t=k\}c_t^\pi\right].
$$
Summing over $t=1,\dots,T$, dividing by $T$, and using the definition of $\bar c_k^\pi$, we obtain
$$
\frac{1}{T}\mathbb{E}\left[\sum_t A_{t,k}^\pi \right]
=
\lambda_k
-
p_k\,\frac{1}{T}\mathbb{E}\!\left[\sum_t \mathbf{1}\{X_t=k\}c_t^\pi\right]
=
\lambda_k-p_k\lambda_k\bar c_k^\pi.
$$
Hence
\begin{equation}\label{eq: stat LB proof, arrival}
    \frac{1}{T}\mathbb{E}\left[\sum_t A_{t,k}^\pi \right] = \lambda_k(1-p_k\bar{c}_k^\pi).
\end{equation}
For the service process, conditional on actually serving a non-empty queue of type $k$, recall
$$
\mathbb{E}[D_{t,k}^\pi \mid U_{t,k}^\pi=1] = \mu_k.
$$
Since $D_{t,k}^\pi=0$ whenever $U_{t,k}^\pi=0$, we can write
$$
\mathbb{E}[D_{t,k}^\pi]
=
\mathbb{E}[D_{t,k}^\pi\mid U_{t,k}^\pi=1]\mathbb{P}(U_{t,k}^\pi=1)
=
\mu_k\,\mathbb{E}[U_{t,k}^\pi].
$$
Therefore,
\begin{equation}\label{eq: stat LB proof, departure}
    \frac{1}{T}\mathbb{E}\left[\sum_t D_{t,k}^\pi \right]=\mu_k\,\frac{1}{T}\mathbb{E}\left[\sum_t U_{t,k}^\pi \right]=\mu_k\bar\rho_k^\pi.
\end{equation}

Putting Equations \ref{eq: stat LB proof, backlog}, \ref{eq: stat LB proof, arrival} and \ref{eq: stat LB proof, departure} into Equation \eqref{eq: average arr, dep, back}, we get
\begin{equation}\label{eq: putting together static quantities}
    \lambda_k(1-p_k\bar{c}_k^\pi) = \mu_k \bar{\rho}_k^\pi+\bar{q}_k^\pi.
\end{equation}
So it follows that policy $\pi$, through $(\bar{\bm{c}}^\pi,\bar{\bm{\rho}}^\pi)$, violates the static constraints of the optimization problem in \eqref{eq: static opt problem} by a term controlled by the final backlog $\bar{\bm{q}}^\pi$.

Let us use a duality argument. Consider the static problem $\text{OPT}(\bm \theta)$ in \eqref{eq: static opt problem}. Introduce the dual multipliers $y_k \geq 0$ for the queue-balance constraints
$$
\lambda_k (1-p_kc_k) - \mu_k \rho_k \leq 0,
$$
and the multiplier $\nu \geq 0$ for the capacity constraint $\sum_{k=1}^K \rho_k \leq 1$. The Lagrangian of the optimization problem is then
$$
\mathcal{L}(\bm{c}, \bm{\rho}, \bm \theta; \bm{y}, \nu) = \sum_k \lambda_kc_k +\sum_k y_k\big(\lambda_k(1-p_kc_k)-\mu_k\rho_k \big) + \nu \big(\sum_k \rho_k -1 \big).
$$
Let $(\bm{y}^*,\nu^*)$ be an optimal dual solution. Since the static linear program in \ref{eq: static opt problem} is feasible and bounded, by strong duality
$$
\text{OPT}(\bm \theta)= \inf_{\bm{c},\bm{\rho}} \mathcal{L}(\bm{c}, \bm{\rho}, \bm \theta; \bm{y}^*, \nu^*).
$$

So for every choice of $(\bm{c},\bm{\rho})$ it holds that
$$
\mathcal{L}(\bm{c}, \bm{\rho}, \bm \theta; \bm{y}^*, \nu^*) \geq \text{OPT}(\bm \theta).
$$
From here, using the definition of $\mathcal{L}$, we get
\begin{equation}\label{eq: inequality from lagrangian}
    \sum_k \lambda_k c_k \geq \text{OPT}(\bm \theta)-\sum_k y_k^*\big(\lambda_k(1-p_kc_k)-\mu_k \rho_k \big)- \nu^*\big(\sum_k \rho_k -1\big).
\end{equation}

Now let us use \eqref{eq: inequality from lagrangian} with $\bm{c}=\bar{\bm{c}}^\pi$ and $\bm{\rho} = \bar{\bm{\rho}}^\pi$. Since $\sum_k \bar{\rho}_k^\pi \leq 1$ and $\nu^* \geq 0$, we have
$$
-\nu^*\big( \sum_k \bar{\rho}_k^\pi - 1\big) \geq 0.
$$
In addition, using Equation \eqref{eq: putting together static quantities}, we get
\begin{equation}\label{eq: final bound with static quantities}
    \sum_k \lambda_k \bar{c}_k^\pi \geq \text{OPT}(\bm \theta)-\sum_k y_k^* \bar{q}_k^\pi.
\end{equation}

We now conclude the lower bound by multiplying \eqref{eq: final bound with static quantities} by $T$ and using the definition of the averages
$$
\mathbb{E}\big[\sum_{t} c_t^\pi \big] \geq T\text{OPT}(\bm \theta)- \sum_k y_k^* \, \mathbb{E}[Q_{T+1,k}^\pi].
$$
Using the definition of total final cost, we get
$$
\mathbb{E}[C^\pi(T)]=\mathbb{E}\big[\sum_t c_t^\pi\big]+ \sum_k r_k \, \mathbb{E}[Q_{T+1,k}^\pi] \geq T\text{OPT}(\bm \theta)+\sum_k (r_k-y_k^*)\mathbb{E}[Q_{T+1,k}^\pi].
$$
Finally, since we assumed that for all $k \in [K]$ $r_k \geq y_k^*$, the second term of the sum in the rhs is non-negative. Hence by taking the infimum over the set of policies we obtain
$$
\inf_{\pi \in \Pi^*} \mathbb{E}[C^\pi(T)] \geq T \cdot \text{OPT}(\bm \theta),
$$
which proves the theorem.
\hfill $\square$

\subsection{Proof of Proposition \ref{propos: first regret bound}}
Since $c_{X_t}^*=c_k^*$ whenever $X_t=k$, and the arrivals are i.i.d. with distribution $\bm\lambda$, it holds that
$$
\frac{1}{T}\sum_{t=1}^T \mathbb{E}[c_{X_t}^*] = \sum_{k=1}^K \lambda_k c_k^* = \text{OPT}(\bm \theta).
$$
Therefore, using Theorem~\ref{thm: static lower bound} in the definition of regret \eqref{eq: regret definition}, we obtain
$$
R_T^\pi \leq \sum_{t=1}^T \mathbb{E}[c_t^\pi] - T\text{OPT}(\bm \theta) + \sum_{k=1}^K r_k \,  \mathbb{E}[Q_{T+1,k}^\pi]
= \sum_{t=1}^T\mathbb{E}[c_t^\pi-c_{X_t}^*]+ \sum_{k=1}^K r_k \, \mathbb{E}[Q_{T+1,k}^\pi].
$$
\hfill $\square$

\section{\texttt{UCB-DPP} Policy}\label{app: UCB-DPP policy}

\subsection{Proof of Lemma \ref{lemma: bound drift DPP preemptive}}
For each $k\in[K]$, the queue dynamics is
$$
Q_{t+1,k}=Q_{t,k}+A_{t,k}-D_{t,k}.
$$
Therefore,
$$
Q_{t+1,k}^2 = Q_{t,k}^2 + (A_{t,k}-D_{t,k})^2 + 2Q_{t,k}(A_{t,k}-D_{t,k}).
$$
Using the definition of $\Phi$, we obtain
$$
\Phi(\bm Q_{t+1})-\Phi(\bm Q_t) =  \frac12\sum_{k=1}^K r_k(A_{t,k}-D_{t,k})^2 + \sum_{k=1}^K r_k Q_{t,k}(A_{t,k}-D_{t,k}).
$$
At each slot there is at most one arrival to the human queues and at most one departure. Hence the vector $(A_{t,k}-D_{t,k})_{k\in[K]}$ has at most two nonzero entries, each with absolute value at most one. Thus,
$$
\sum_{k=1}^K (A_{t,k}-D_{t,k})^2 \le 2.
$$
Since $r_k\le r_{\max}$ for all $k\in[K]$, it follows that
$$
\frac12\sum_{k=1}^K r_k(A_{t,k}-D_{t,k})^2 \le r_{\max}.
$$
Consequently,
$$
\Phi(\bm Q_{t+1})-\Phi(\bm Q_t) \le r_{\max} + \sum_{k=1}^K r_k Q_{t,k}(A_{t,k}-D_{t,k}).
$$
Taking conditional expectation with respect to $\mathcal F_t$ gives the desired bound.
\hfill $\square$

The \texttt{UCB-DPP} policy is summarized in Algorithm \ref{alg:ucb-dpp}. At each slot, the policy first updates boosted estimates of the unknown chatbot and service parameters, then minimizes the optimistic drift-plus-penalty surrogate. Due to the linearity of the chatbot success probability in the cost, the chatbot decision has a threshold form.

\begin{algorithm}[H]
\caption{\texttt{UCB-DPP} Policy}
\label{alg:ucb-dpp}
\begin{algorithmic}[1]
\State \textbf{Input:} horizon $T$, number of classes $K$, weights $(r_k)_{k=1}^K$, parameter $V>0$, confidence level $\delta\in(0,1)$.
\State \textbf{Initialize:} $Q_{1,k}=0$ and $N_k(0)=S_k(0)=M_k(0)=G_k(0)=0$ for all $k\in[K]$.
\For{$t=1,\ldots,T$}
    \State Observe the current queues $\bm Q_t$ and the arriving type $X_t$.
    \State For each $k\in[K]$, compute $\hat\mu_k(t-1)$ as in \eqref{eq: estimator of mu} and $\bar\mu_k(t-1)$ as in \eqref{eq: mu_ucb}.
    \State For each $k\in[K]$, compute $\hat p_k(t-1)$ as in \eqref{eq: estimator of p} and $\bar p_k(t-1)$ as in \eqref{eq: p_ucb}.
    \State Choose the \emph{chatbot cost}
    $$
    c_t=
    \begin{cases}
    1, & \text{if } r_{X_t}Q_{t,X_t}\bar p_{X_t}(t-1)\ge V,\\
    0, & \text{otherwise.}
    \end{cases}
    $$

    \State Choose the \emph{scheduling action}
    $$
    a_t\in\arg\max_{k\in[K]} r_k Q_{t,k}\bar\mu_k(t-1).
    $$
    \State If all queues are empty, choose $a_t$ arbitrarily.

    \State Apply the chatbot decision and observe: $Y_t=\mathds{1}\{\text{chatbot succeeds at slot }t\}$.
    \State Define the human-queue arrivals: $A_{t,k}=\mathds{1}\{X_t=k\}(1-Y_t)$ for all $k\in[K]$.

    \State \emph{Serve} queue $a_t$. If $Q_{t,a_t}>0$, observe the service outcome $D_{t,a_t}$; otherwise set $D_{t,a_t}=0$.
    \State \emph{Set} $D_{t,k}=0$ for all $k\neq a_t$.

    \State \emph{Update} the queues according to:
    $$
    Q_{t+1,k}=Q_{t,k}+A_{t,k}-D_{t,k}, \qquad k\in[K].
    $$
    \State \emph{Update} the chatbot counters:
    $$
    M_{X_t}(t)=M_{X_t}(t-1)+c_t, \qquad G_{X_t}(t)=G_{X_t}(t-1)+Y_t.
    $$
    \State \emph{Update} the service counters:
    $$
    N_{a_t}(t)=N_{a_t}(t-1)+\mathds{1}\{Q_{t,a_t}>0\}, \qquad S_{a_t}(t)=S_{a_t}(t-1)+D_{t,a_t}.
    $$
    \State All counters not explicitly updated remain unchanged.
\EndFor
\end{algorithmic}
\end{algorithm}

\clearpage

\section{Regret Bound}\label{app: all regret bound}

\subsection{Proof of Regret Decomposition Lemma \ref{lemma: regret decomposition}} \label{app: regret decomposition}

Recall that, according to Proposition~\ref{propos: first regret bound}, the regret of any policy, and in particular that of \texttt{UCB-DPP}, satisfies
$$
R_T^{\texttt{UCB-DPP}} \leq \sum_{t=1}^T \mathbb{E}[c_t^{\texttt{UCB-DPP}}-c_{X_t}^*] + \sum_{k=1}^K r_k\,\mathbb{E}[Q_{T+1,k}^{\texttt{UCB-DPP}}].
$$
For ease of notation, from now on we omit the superscript \texttt{UCB-DPP} from all relevant quantities whenever no ambiguity arises. In particular, we simply write $c_t$, $a_t$, and $Q_{t,k}$.
Recall that $(\bm c^*,\bm \rho^*)$ is the optimal solution of the static optimization problem \eqref{eq: static opt problem}.

\paragraph{Decomposition into good and bad events.}
For every $t \in [T]$,
$$
\mathbb E[c_t-c_{X_t}^*] = \mathbb E[(c_t-c_{X_t}^*)\mathds{1}\{\mathscr G_t\}] + \mathbb E[(c_t -c_{X_t}^*)\mathds{1}\{\overline{\mathscr G}_t\}].
$$
Since both $c_t$ and $c_{X_t}^*$ are in $[0,1]$, it follows that
$$
c_t-c_{X_t}^* \leq 1.
$$
Hence
\begin{equation}
\mathbb E[c_t-c_{X_t}^*] \le \mathbb E[(c_t-c_{X_t}^*)\mathds{1}\{\mathscr G_t\}] + \mathbb P(\overline{\mathscr G}_t).
\end{equation}

We can now observe that $\overline{\mathscr G}_t\subseteq \overline{\mathscr G}_T$, hence
$$
\mathbb P(\overline{\mathscr G}_t)\le \mathbb P(\overline{\mathscr G}_T).
$$
Summing over $t$, we obtain:
\begin{equation}\label{eq: good bad decomposition both unknown}
\sum_{t=1}^T \mathbb E[c_t-c_{X_t}^*]
\le
\sum_{t=1}^T \mathbb E[(c_t-c_{X_t}^*)\mathds{1}\{\mathscr G_t\}]
+
T\,\mathbb P(\overline{\mathscr G}_T).
\end{equation}
The second term is exactly $R_2(T)$.

\paragraph{Drift Plus Penalty under good events.}
Let $\mathcal H_t$ denote the information available at the beginning of slot $t$, before the arrival type $X_t$ is observed. By the i.i.d.\ assumption on the arrivals, for every $k\in[K]$,
$$
\mathbb P(X_t=k\mid \mathcal H_t)=\lambda_k.
$$

By the bound on the conditional drift in Lemma \ref{lemma: bound drift DPP preemptive}, together with the conditional expectations on arrivals and departures, we have
\begin{equation}
    \Delta_t+Vc_t \le r_{\max}+r_{X_t}Q_{t,X_t}(1-p_{X_t}c_t)-r_{a_t}Q_{t,a_t}\mu_{a_t} +Vc_t.
\end{equation}
Let us add and subtract the quantities $r_{X_t}Q_{t,X_t}(1-\bar p_{X_t}(t-1)c_t)$ and $r_{a_t}Q_{t,a_t}\bar\mu_{a_t}(t-1)$.
We obtain
\begin{equation*}
\begin{aligned}
\Delta_t+Vc_t \leq & \;r_{\max} + \left(Vc_t+r_{X_t}Q_{t,X_t}(1-\bar p_{X_t}(t-1)c_t)-r_{a_t}Q_{t,a_t}\bar\mu_{a_t}(t-1) \right)\\
& + r_{X_t}Q_{t,X_t}c_t\left(\bar p_{X_t}(t-1)-p_{X_t}\right) + r_{a_t}Q_{t,a_t}\left(\bar\mu_{a_t}(t-1)-\mu_{a_t}\right).
\end{aligned}
\end{equation*}

By policy definition, \texttt{UCB-DPP} chooses at each $t$ the pair $(c_t,a_t)$ minimizing
$$
Vc+r_{X_t}Q_{t,X_t}(1-\bar p_{X_t}(t-1)c)-r_aQ_{t,a}\bar\mu_a(t-1),
$$
over $c\in[0,1]$ and $a\in\{1,\ldots,K\}$. Therefore, for every alternative pair $(\tilde c_t,\tilde a_t)$,
\begin{equation}\label{eq: objective bound with tildes both unknown}
\begin{aligned}
\Delta_t+Vc_t \leq & \; r_{\max} + \left( V\tilde c_t+r_{X_t}Q_{t,X_t}(1-\bar p_{X_t}(t-1)\tilde c_t)-r_{\tilde a_t}Q_{t,\tilde a_t}\bar\mu_{\tilde a_t}(t-1) \right)\\
& + r_{X_t}Q_{t,X_t}c_t\left(\bar p_{X_t}(t-1)-p_{X_t}\right) + r_{a_t}Q_{t,a_t}\left(\bar\mu_{a_t}(t-1)-\mu_{a_t}\right).
\end{aligned}
\end{equation}

\paragraph{Comparison with the static benchmark.}
We now compare the policy with the static solution $(\bm c^*,\bm\rho^*)$. Choose $(\tilde c_t,\tilde a_t)$ in Equation \eqref{eq: objective bound with tildes both unknown} as follows:
\begin{itemize}
    \item if $X_t=k$, then $\tilde c_t=c_k^*$;
    \item choose $\tilde a_t=j$ with probability $\rho_j^*$.
\end{itemize}
Note that this randomized scheduling rule, induced by the solution of the static optimization problem, is not queue-aware: if the selected queue is empty, then no service takes place during that slot, so the policy effectively wastes one service opportunity.

Substituting $\tilde c_t=c_{X_t}^*$ into \eqref{eq: objective bound with tildes both unknown} and taking conditional expectation with respect to the randomized comparison action $\tilde a_t$, given the history up to the beginning of slot $t$ and the realized type $X_t$, we first obtain
\begin{align*}
\Delta_t+V(c_t-c_{X_t}^*) \leq &\; r_{\max} + r_{X_t}Q_{t,X_t}\left(1-\bar p_{X_t}(t-1)c_{X_t}^*\right) - \sum_{j=1}^K \rho_j^*\, r_jQ_{t,j}\bar\mu_j(t-1)\\
&+ r_{X_t}Q_{t,X_t}c_t\left(\bar p_{X_t}(t-1)-p_{X_t}\right) + r_{a_t}Q_{t,a_t}\left(\bar\mu_{a_t}(t-1)-\mu_{a_t}\right).
\end{align*}
Here we used that
$$
\mathbb{E}\!\left[r_{\tilde a_t}Q_{t,\tilde a_t}\bar\mu_{\tilde a_t}(t-1)\mid \mathcal F_t\right]
=
\sum_{j=1}^K \rho_j^*\, r_jQ_{t,j}\bar\mu_j(t-1).
$$

We then take conditional expectation with respect to $X_t$, conditioning on the information available at the beginning of slot $t$. Since $\mathbb{P}(X_t=k\mid\mathcal H_t)=\lambda_k$, we obtain
\begin{align*}
\mathbb E\!\left[\Delta_t+V(c_t-c_{X_t}^*) \mid \mathcal H_t\right]
\leq & \; r_{\max} + \sum_{k=1}^K r_kQ_{t,k} \left( \lambda_k(1-\bar p_k(t-1)c_k^*) - \rho_k^*\bar\mu_k(t-1) \right)\\
&+\mathbb E\!\left[r_{X_t}Q_{t,X_t}c_t\left(\bar p_{X_t}(t-1)-p_{X_t}\right)\mid \mathcal H_t\right]\\
&+\mathbb E\!\left[r_{a_t}Q_{t,a_t}\left(\bar\mu_{a_t}(t-1)-\mu_{a_t}\right)\mid \mathcal H_t\right].
\end{align*}

Now suppose that the good event $\mathscr G_t$ holds. Then, for every $k\in[K]$,
$$
\bar p_k(t-1)\geq p_k \qquad \text{and} \qquad \bar\mu_k(t-1)\geq \mu_k.
$$
Therefore,
$$
\lambda_k(1-\bar p_k(t-1)c_k^*)- \rho_k^*\bar\mu_k(t-1) \leq \lambda_k(1-p_kc_k^*) - \rho_k^*\mu_k.
$$
Since $(\bm c^*,\bm\rho^*)$ is feasible for the static problem,
$$
\lambda_k(1-p_kc_k^*) - \rho_k^*\mu_k \le 0.
$$
Hence, under $\mathscr G_t$,
\begin{equation}
\begin{aligned}
\mathbb E\!\left[\Delta_t+V(c_t-c_{X_t}^*) \mid \mathcal H_t\right] \leq &\; 
r_{\max}+ \mathbb E\!\left[r_{X_t}Q_{t,X_t}c_t\left(\bar p_{X_t}(t-1)-p_{X_t}\right)\mid \mathcal H_t\right] \\
&+\mathbb E\!\left[r_{a_t}Q_{t,a_t}\left(\bar\mu_{a_t}(t-1)-\mu_{a_t}\right)\mid \mathcal H_t\right].
\end{aligned}
\end{equation}

Multiplying by $\mathds{1}\{\mathscr G_t\}$, taking expectation, and summing over $t$, we obtain
\begin{equation}\label{eq: final bound on c difference both unknown}
\begin{aligned}
\sum_{t=1}^T \mathbb E[(c_t-c_{X_t}^*)\mathds{1}\{\mathscr G_t\}] \leq & \; \frac{r_{\max}T}{V}+\frac{1}{V}\mathbb E\!\left[\sum_{t=1}^T r_{X_t}Q_{t,X_t}c_t\left(\bar p_{X_t}(t-1)-p_{X_t}\right)\mathds{1}\{\mathscr G_t\} \right]\\
&+\frac{1}{V}\mathbb E\!\left[\sum_{t=1}^Tr_{a_t}Q_{t,a_t}\left(\bar\mu_{a_t}(t-1)-\mu_{a_t}\right)\mathds{1}\{\mathscr G_t\}\right]\\
&-\frac{1}{V}\sum_{t=1}^T \mathbb E[\mathds{1}\{\mathscr G_t\}\Delta_t].
\end{aligned}
\end{equation}

\paragraph{Control of the drift term.}
Using the definition of $\Delta_t$ and linearity of expectation,
$$
\sum_{t=1}^T \mathbb E[\mathds{1}\{\mathscr G_t\}\Delta_t]=\mathbb E\!\left[\sum_{t=1}^T\mathds{1}\{\mathscr G_t\}\left(\Phi(\bm Q_{t+1})-\Phi(\bm Q_t)\right)\right].
$$
Since the events $\{\mathscr G_t\}_{t\in[T]}$ are decreasing and nested, we have
$$
\sum_{t=1}^T\mathds{1}\{\mathscr G_t\}\left(\Phi(\bm Q_{t+1})-\Phi(\bm Q_t)\right)\geq-\Phi(\bm Q_1)=0,
$$
because $\bm Q_1=\bm 0$. Therefore,
$$
-\sum_{t=1}^T \mathbb E[\mathds{1}\{\mathscr G_t\}\Delta_t]\le 0.
$$

Plugging this into Equation \eqref{eq: final bound on c difference both unknown}, we finally get
\begin{equation}\label{eq: regret decomposition final bound without Delta_t}
\begin{aligned}
\sum_{t=1}^T \mathbb E[(c_t-c_{X_t}^*)\mathds{1}\{\mathscr G_t\}] \leq & \; \frac{r_{\max}T}{V}+\frac{1}{V}\mathbb E\!\left[\sum_{t=1}^T
r_{X_t}Q_{t,X_t}c_t\left(\bar p_{X_t}(t-1)-p_{X_t}\right)\mathds{1}\{\mathscr G_t\}\right]\\
&+\frac{1}{V}\mathbb E\!\left[\sum_{t=1}^Tr_{a_t}Q_{t,a_t}\left(\bar\mu_{a_t}(t-1)-\mu_{a_t}\right)\mathds{1}\{\mathscr G_t\}\right].
\end{aligned}
\end{equation}

The three terms on the right-hand side are exactly $R_3(T)$, $R_1^p(T)$, and $R_1^\mu(T)$.

\paragraph{Conclusion.}
By Proposition \ref{propos: first regret bound},
$$
R_T^{\texttt{UCB-DPP}} \leq \sum_{t=1}^T \mathbb E[c_t-c_{X_t}^*] + \sum_{k=1}^K r_k\,\mathbb E[Q_{T+1,k}].
$$
Combining this with Equations \eqref{eq: good bad decomposition both unknown} and \ref{eq: regret decomposition final bound without Delta_t}, we obtain
$$
R_T^{\texttt{UCB-DPP}} \leq R_1^p(T)+R_1^\mu(T)+R_2(T)+R_3(T)+R_4(T),
$$
where the last term is exactly
$$
R_4(T)=\sum_{k=1}^K r_k\,\mathbb E[Q_{T+1,k}].
$$
This proves the lemma.
\hfill $\square$

\subsection{Proof of Proposition \ref{propos: bound on bad event proba}} \label{app: bound on bad event proba}

We split the good event as defined in Equation \eqref{eq: good event} into the service-side and chatbot-side parts. Define
$$
\mathscr G_T^\mu := \bigcap_{s=1}^T \bigcap_{k=1}^K \left\{ |\hat\mu_k(s-1)-\mu_k|\le \beta_k^\mu(s-1)\right\},
$$
and
$$
\mathscr G_T^p := \bigcap_{s=1}^T \bigcap_{k=1}^K \left\{ |\hat p_k(s-1)-p_k|\leq \beta_k^p(s-1) \right\}.
$$
Then
$$
\mathscr G_T=\mathscr G_T^\mu \cap \mathscr G_T^p,
$$
so that
$$
\mathbb P(\overline{\mathscr G}_T) \le \mathbb P(\overline{\mathscr G}_T^\mu)+\mathbb P(\overline{\mathscr{G}}_T^p).
$$

\paragraph{Control of the service-side event.}
For all $k \in [K]$ and $n \geq 1$, define the time of the $n$-th effective service attempt on type $k$ by
$$
\tau_{n,k} := \inf \{t \geq 1: N_k(t)=n\}.
$$
When $\tau_{n,k} < \infty$, define
$$
d_{n,k} := D_{\tau_{n,k},k},
$$
namely the indicator of completion observed at the $n$-th effective service attempt on type $k$.

By the model assumptions, every time the server makes an effective service attempt on a queue of type $k$, the completion occurs with probability $\mu_k$, independently of the past. Therefore, for every $k$, the sequence $(d_{n,k})_{n\ge 1}$ is i.i.d.\ Bernoulli with parameter $\mu_k$.

Moreover, for every $t$,
$$
S_k(t)=\sum_{n=1}^{N_k(t)} d_{n,k}.
$$
Hence, if $N_k(t)=n\ge 1$, then
$$
\hat\mu_k(t)=\frac{1}{n}\sum_{m=1}^n d_{m,k}.
$$

Now fix $s\in\{1,\dots,T\}$ and $k\in[K]$. If $N_k(s-1)=0$, then by definition
$$
\hat\mu_k(s-1)=0, \qquad \beta_k^\mu(s-1)=1,
$$
and therefore
$$
|\hat\mu_k(s-1)-\mu_k|\le 1=\beta_k^\mu(s-1).
$$

Assume now that $N_k(s-1)\ge 1$. Since $N_k(s-1)\le s-1$, we can write
\begin{equation*}
\mathbb P\!\left(|\hat\mu_k(s-1)-\mu_k|>\beta_k^\mu(s-1)\right) =
\sum_{n=1}^{s-1}
\mathbb P\!\left(
|\hat\mu_k(s-1)-\mu_k|>\beta_k^\mu(s-1),\ N_k(s-1)=n
\right).
\end{equation*}
On the event $\{N_k(s-1)=n\}$, we have
$$
\hat\mu_k(s-1)=\frac{1}{n}\sum_{m=1}^n d_{m,k},
\qquad
\beta_k^\mu(s-1)=\sqrt{\frac{\log\!\left(4K(s-1)/\delta\right)}{2n}}.
$$
Therefore,
\begin{equation*}
\mathbb P\!\left(|\hat\mu_k(s-1)-\mu_k|>\beta_k^\mu(s-1)\right) \le
\sum_{n=1}^{s-1}
\mathbb P\!\left(
\left|\frac{1}{n}\sum_{m=1}^n d_{m,k}-\mu_k\right|
>
\sqrt{\frac{\log\!\left(4K(s-1)/\delta\right)}{2n}}
\right).
\end{equation*}
By Hoeffding's inequality, for every $n\ge 1$,
$$
\mathbb P\!\left(
\left|\frac{1}{n}\sum_{m=1}^n d_{m,k}-\mu_k\right|
>
\sqrt{\frac{\log\!\left(4K(s-1)/\delta\right)}{2n}}
\right)
\le
\frac{\delta}{2K(s-1)}.
$$
Hence,
$$
\mathbb P\!\left(|\hat\mu_k(s-1)-\mu_k|>\beta_k^\mu(s-1)\right)
\le
\sum_{n=1}^{s-1}\frac{\delta}{2K(s-1)}
=
\frac{\delta}{2K}.
$$

Taking a union bound over $s\in\{2,\dots,T\}$ and $k\in[K]$, we get
$$
\mathbb P(\overline{\mathscr G}_T^\mu)
\le
\sum_{s=2}^T\sum_{k=1}^K \frac{\delta}{2K}
=
\frac{(T-1)\delta}{2}.
$$

\paragraph{Control of the chatbot-side event.}
Recall that the chatbot decision is bang-bang, so the estimation of $p_k$ is based on those slots in which $X_t=k$ and $c_t=1$.

For all $k \in [K]$ and $n \geq 1$, define the time of the $n$-th informative chatbot exposure on type $k$ by
$$
\sigma_{n,k}:=\inf\{t\ge 1:\ M_k(t)=n\}.
$$
When $\sigma_{n,k}<\infty$, define
$$
z_{n,k}:=Y_{\sigma_{n,k}},
$$
where $Y_t=\mathds{1}\{\text{chatbot succeeds at slot }t\}$.

By construction, whenever $t=\sigma_{n,k}$ we have $X_t=k$ and $c_t=1$, hence
$$
\mathbb E[Y_t\mid X_t=k,c_t=1]=p_k.
$$
Therefore, for every $k$, the sequence $(z_{n,k})_{n\ge 1}$ is i.i.d.\ Bernoulli with parameter $p_k$.

Moreover, for every $t$,
$$
G_k(t)=\sum_{n=1}^{M_k(t)} z_{n,k}.
$$
Hence, if $M_k(t)=n\ge 1$, then
$$
\hat p_k(t)=\frac{1}{n}\sum_{m=1}^n z_{m,k}.
$$

Now fix $s\in\{1,\dots,T\}$ and $k\in[K]$. If $M_k(s-1)=0$, then by definition
$$
\hat p_k(s-1)=0,
\qquad
\beta_k^p(s-1)=1,
$$
and therefore
$$
|\hat p_k(s-1)-p_k|\le 1=\beta_k^p(s-1).
$$

Assume now that $M_k(s-1)\ge 1$. As in the service-side case, we decompose according to the possible values of $M_k(s-1)$ and use
$$
\hat p_k(s-1)=\frac{1}{n}\sum_{m=1}^n z_{m,k}
\qquad
\text{on } \{M_k(s-1)=n\},
$$
Therefore,
\begin{align*}
\mathbb P\!\left(|\hat p_k(s-1)-p_k|>\beta_k^p(s-1)\right)
&\le
\sum_{n=1}^{s-1}
\mathbb P\!\left(
\left|\frac{1}{n}\sum_{m=1}^n z_{m,k}-p_k\right|
>
\sqrt{\frac{\log\!\left(4K(s-1)/\delta\right)}{2n}}
\right) \\
&\le
\sum_{n=1}^{s-1}\frac{\delta}{2K(s-1)}
=
\frac{\delta}{2K},
\end{align*}
where the second inequality follows from Hoeffding's inequality.

Taking a union bound over $s\in\{2,\dots,T\}$ and $k\in[K]$, we obtain
$$
\mathbb P(\overline{\mathscr G}_T^p)
\le
\sum_{s=2}^T\sum_{k=1}^K \frac{\delta}{2K} = \frac{(T-1)\delta}{2}.
$$

\paragraph{Conclusion.}
Combining the bounds on the two parts,
$$
\mathbb P(\overline{\mathscr G}_T)
\le
\mathbb P(\overline{\mathscr G}_T^\mu)+\mathbb P(\overline{\mathscr G}_T^p)
\le
\delta(T-1) \leq \delta T.
$$
This completes the proof.
\hfill $\square$

\subsection{Bounding $R_4(T)$}

Before proceeding, we introduce several auxiliary quantities that will be used throughout the proofs. Set
\begin{equation}\label{eq: A_epsilon}
    A_\varepsilon:=r_{\max}+V\sum_{k=1}^K \lambda_k c_k^\varepsilon,
\end{equation}
and
\begin{equation}\label{eq: L,eta,delta,C}
L_\varepsilon:=\left\lceil \frac{32\log(4KT/\delta)}{\varepsilon^2}\right\rceil,
\qquad
\eta_\varepsilon:=\frac{\varepsilon\sqrt{r_{\min}}}{8},
\qquad
\delta_r:=\sqrt{2r_{\max}},
\qquad
C_r:=\delta_r+\eta_\varepsilon.
\end{equation}

Finally, define
\begin{equation}\label{eq: theta_epsilon}
\theta_\varepsilon:=\frac{8A_\varepsilon}{\varepsilon\sqrt{r_{\min}}}.
\end{equation}

We also define, for each $t\in[T]$, the indicators of exceptional slots:
\begin{equation}\label{eq: exceptionl slots mu}
    J_t^\mu:=\mathds{1}\left\{Q_{t,a_t}>0,\ \beta^\mu_{a_t}(t-1)>\frac{\varepsilon}{8}\right\},
\end{equation}
and
\begin{equation}\label{eq: exceptionl slots p}
    J_t^p:=\mathds{1}\left\{c_t=1,\ \beta^p_{X_t}(t-1)>\frac{\varepsilon}{8}\right\}.
\end{equation}

These indicators identify the slots in which the relevant confidence radius is larger than a threshold proportional to $\varepsilon$. We finally define
\begin{equation}\label{eq: joint exceptionl slots}
    J_t:=\max\{J_t^\mu, J_t^p\}.
\end{equation}

We then state a well-known drift lemma introduced in \cite{yu2017online}. For our purposes, we present it directly in the special case $t_0=1$.
\begin{lemma}\label{lemma:hajek_drift}
Let $\{U_t\}_{t\ge 1}$ be a process adapted to a filtration $\{\mathcal F_t\}_{t\ge 1}$, with $U_1=0$. Assume that there exist constants $\theta>0$, $\delta>0$, and $\eta>0$ such that, for all $t\ge 1$,
$$
U_{t+1}-U_t\le \delta
\qquad \text{a.s.},
$$
and
$$
\mathbb E[U_{t+1}-U_t\mid \mathcal F_t]\le \delta
\qquad \text{whenever } U_t<\theta,
$$
while
$$
\mathbb E[U_{t+1}-U_t\mid \mathcal F_t]\le -\eta
\qquad \text{whenever } U_t\ge \theta.
$$
Then, for every $t\ge 1$,
$$
\mathbb E[U_t]
\le
\theta+\delta+\frac{4\delta^2}{\eta}\log\!\left(\frac{8\delta^2}{\eta^2}\right).
$$
\end{lemma}

\subsubsection{Proof of Lemma \ref{lemma:deterministic_weighted_backlog}} \label{app: proof of det_weighted_backlog}

Recall that $\mathcal H_t$ denote the information available at the beginning of slot $t$, before the arrival type $X_t$ is observed. In particular, $\bm Q_t$, $\hat{\bm p}(t-1)$, $\hat{\bm\mu}(t-1)$, and $\mathscr G_t$ are $\mathcal H_t$-measurable. Recall also that $\mathcal F_t$ is the information available after observing the arrival type and selecting the control actions at slot $t$.

We first control the number of exceptional slots as defined in Equations \ref{eq: exceptionl slots mu}, \ref{eq: exceptionl slots p} and \ref{eq: joint exceptionl slots}.

\paragraph{Service-side exceptional slots.}
If $J_t^\mu=1$ and $a_t=k$, then either $N_k(t-1)=0$, in which case trivially
$$
N_k(t-1)=0<\frac{32\log(4KT/\delta)}{\varepsilon^2},
$$
or $N_k(t-1)\ge 1$, in which case necessarily $t\ge 2$ and
$$
\beta_k^\mu(t-1) = \sqrt{\frac{\log(4K(t-1)/\delta)}{2N_k(t-1)}} > \frac{\varepsilon}{8},
$$
which implies
$$
N_k(t-1) < \frac{32\log(4K(t-1)/\delta)}{\varepsilon^2} \le \frac{32\log(4KT/\delta)}{\varepsilon^2}.
$$
Since $Q_{t,k}>0$, slot $t$ is an effective service attempt on type $k$, hence the counter $N_k$ increases by one.  
More precisely, for each fixed $k\in[K]$, every slot $t$ such that $J_t^\mu=1$ and $a_t=k$ contributes one unit to the counter $N_k$. Since this can only occur while $N_k(t-1)<L_\varepsilon$ (as defined in Equation \ref{eq: L,eta,delta,C}), we must have
$$
\sum_{t=1}^T \mathds{1}\{J_t^\mu=1,\ a_t=k\}\le L_\varepsilon.
$$
Summing over $k\in[K]$, we obtain
\begin{equation}\label{eq: exceptional_slots_mu_bound}
\sum_{t=1}^T J_t^\mu = \sum_{k=1}^K \sum_{t=1}^T \mathds{1}\{J_t^\mu=1,\ a_t=k\} \le K L_\varepsilon.
\end{equation}

\paragraph{Chatbot-side exceptional slots.}
We can reason in the same way as in the previous case. If $J_t^p=1$ and $X_t=k$, then by definition $c_t=1$, so slot $t$ is an informative chatbot exposure for type $k$, and therefore the counter $M_k$ increases by one at time $t$.

Moreover, either $M_k(t-1)=0$, in which case trivially
$$
M_k(t-1)=0<\frac{32\log(4KT/\delta)}{\varepsilon^2},
$$
or $M_k(t-1)\ge 1$, in which case necessarily $t\ge 2$ and
$$
\beta_k^p(t-1) = \sqrt{\frac{\log(4K(t-1)/\delta)}{2M_k(t-1)}} > \frac{\varepsilon}{8},
$$
which implies
$$
M_k(t-1) < \frac{32\log(4K(t-1)/\delta)}{\varepsilon^2} \le \frac{32\log(4KT/\delta)}{\varepsilon^2}.
$$
Hence, for each fixed $k$, the event $\{J_t^p=1,\ X_t=k\}$ can occur at most $L_\varepsilon$ times. Summing over $k$, we obtain the bound
\begin{equation}\label{eq: exceptional_slots_p_bound}
\sum_{t=1}^T J_t^p \le K L_\varepsilon.
\end{equation}

Since $J_t \le J_t^\mu+J_t^p$, it follows that
\begin{equation}\label{eq: exceptional_slots_total_bound}
\sum_{t=1}^T J_t \le 2K L_\varepsilon.
\end{equation}

\paragraph{Drift bound on good non-exceptional slots.}
By Lemma \ref{lemma: bound drift DPP preemptive},
$$
\Delta_t \le r_{\max}+r_{X_t}Q_{t,X_t}(1-p_{X_t}c_t)-r_{a_t}Q_{t,a_t}\mu_{a_t}.
$$
Hence, by adding and subtracting relevant quantities,
\begin{align*}
\Delta_t+Vc_t \le & \;r_{\max} + \left(Vc_t+r_{X_t}Q_{t,X_t}(1-\bar p_{X_t}(t-1)c_t)-r_{a_t}Q_{t,a_t}\bar\mu_{a_t}(t-1) \right) \\
& + r_{X_t}Q_{t,X_t}c_t\left(\bar p_{X_t}(t-1)-p_{X_t}\right)
+ r_{a_t}Q_{t,a_t}\left(\bar\mu_{a_t}(t-1)-\mu_{a_t}\right).
\end{align*}

By policy definition, for every realized $X_t$, the \texttt{UCB-DPP} policy minimizes in $(c,a)$ the surrogate objective
$$
Vc+r_{X_t}Q_{t,X_t}(1-\bar p_{X_t}(t-1)c)-r_aQ_{t,a}\bar\mu_a(t-1).
$$

We now compare the policy with the randomized stationary policy induced by $(\bm c^\varepsilon,\bm\rho^\varepsilon)$, whose existence is guaranteed by assumption. More precisely, we consider the policy defined as follows:
\begin{itemize}
    \item choose $\tilde c_t=c_{X_t}^\varepsilon$;
    \item choose $\tilde a_t=k$ with probability $\rho_k^\varepsilon$.
\end{itemize}

Hence, by taking conditional expectation first with respect to the randomized comparison action and then with respect to $X_t$, conditioning on $\mathcal H_t$, we get
\begin{align*}
\mathbb E[\Delta_t+Vc_t\mid \mathcal H_t] \le & \; A_\varepsilon
+ \sum_{k=1}^K r_kQ_{t,k}
\left(
\lambda_k(1-\bar p_k(t-1)c_k^\varepsilon)-\rho_k^\varepsilon\bar\mu_k(t-1)
\right) \\
&+ \mathbb E\!\left[ r_{X_t}Q_{t,X_t}c_t\left(\bar p_{X_t}(t-1)-p_{X_t}\right) \,\middle|\, \mathcal H_t \right] \\
&+ \mathbb E\!\left[ r_{a_t}Q_{t,a_t}\left(\bar\mu_{a_t}(t-1)-\mu_{a_t}\right) \,\middle|\, \mathcal H_t \right].
\end{align*}
where $A_\varepsilon$ is defined in Equation \ref{eq: A_epsilon}.

Now let us work on the event $\mathscr G_t\cap\{J_t=0\}$. Since this event is $\mathcal F_t$-measurable, we first control the last two terms pathwise on $\mathcal F_t$, and then return to $\mathcal H_t$ by the tower property.

Since $\mathscr G_t$ holds and $J_t^\mu=0$, we have
$$
0\le \bar\mu_{a_t}(t-1)-\mu_{a_t}\le 2\beta^\mu_{a_t}(t-1)\le \frac{\varepsilon}{4},
$$
and therefore, pathwise on $\mathscr G_t\cap\{J_t=0\}$,
$$
r_{a_t}Q_{t,a_t}\left(\bar\mu_{a_t}(t-1)-\mu_{a_t}\right)
\le
\frac{\varepsilon}{4}\sum_{k=1}^K r_kQ_{t,k}.
$$

Moreover, on $\mathscr G_t\cap\{J_t=0\}$, the chatbot optimism term is controlled pathwise as follows:
\begin{itemize}
    \item if $c_t=0$, then
    $$
    r_{X_t}Q_{t,X_t}c_t\left(\bar p_{X_t}(t-1)-p_{X_t}\right)=0;
    $$
    \item if $c_t=1$, then $J_t^p=0$, hence
    $$
    \beta^p_{X_t}(t-1)\le \frac{\varepsilon}{8},
    $$
    and since $\mathscr G_t$ holds,
    $$
    0\le \bar p_{X_t}(t-1)-p_{X_t}\le 2\beta^p_{X_t}(t-1)\le \frac{\varepsilon}{4}.
    $$
    Therefore,
    $$
    r_{X_t}Q_{t,X_t}c_t\left(\bar p_{X_t}(t-1)-p_{X_t}\right)
    \le
    \frac{\varepsilon}{4}\,r_{X_t}Q_{t,X_t}
    \le
    \frac{\varepsilon}{4}\sum_{k=1}^K r_kQ_{t,k}.
    $$
\end{itemize}

Thus, by the tower property, 
\begin{align*}
&\mathbb E\!\left[
r_{a_t}Q_{t,a_t}\left(\bar\mu_{a_t}(t-1)-\mu_{a_t}\right)
\mathds{1}\{\mathscr G_t\}\mathds{1}\{J_t=0\}
\,\middle|\, \mathcal H_t
\right] \\
&\qquad \le
\frac{\varepsilon}{4}\sum_{k=1}^K r_kQ_{t,k}\,
\mathbb P(\mathscr G_t\cap\{J_t=0\}\mid \mathcal H_t),
\end{align*}
and similarly
\begin{align*}
&\mathbb E\!\left[
r_{X_t}Q_{t,X_t}c_t\left(\bar p_{X_t}(t-1)-p_{X_t}\right)
\mathds{1}\{\mathscr G_t\}\mathds{1}\{J_t=0\}
\,\middle|\, \mathcal H_t
\right] \\
&\qquad \le
\frac{\varepsilon}{4}\sum_{k=1}^K r_kQ_{t,k}\,
\mathbb P(\mathscr G_t\cap\{J_t=0\}\mid \mathcal H_t).
\end{align*}

Finally, since $\mathscr G_t$ implies $\bar p_k(t-1)\ge p_k$ and $\bar\mu_k(t-1)\ge \mu_k$, strict feasibility gives
$$
\lambda_k(1-\bar p_k(t-1)c_k^\varepsilon)-\rho_k^\varepsilon\bar\mu_k(t-1)
\le
\lambda_k(1-p_kc_k^\varepsilon)-\rho_k^\varepsilon\mu_k
\le -\varepsilon.
$$
Therefore,
\begin{align*}
&\mathbb E\!\left[(\Delta_t+Vc_t)\mathds{1}\{\mathscr G_t\}\mathds{1}\{J_t=0\}\mid \mathcal H_t\right] \\
&\qquad \le
\left(
A_\varepsilon
-\varepsilon\sum_{k=1}^K r_kQ_{t,k}
+\frac{\varepsilon}{4}\sum_{k=1}^K r_kQ_{t,k}
+\frac{\varepsilon}{4}\sum_{k=1}^K r_kQ_{t,k}
\right)
\mathbb P(\mathscr G_t\cap\{J_t=0\}\mid \mathcal H_t),
\end{align*}
that is,
\begin{equation}\label{eq: drift_good_non_exceptional_weighted}
\mathbb E\!\left[(\Delta_t+Vc_t)\mathds{1}\{\mathscr G_t\}\mathds{1}\{J_t=0\}\mid \mathcal H_t\right]
\le
\left(
A_\varepsilon-\frac{\varepsilon}{2}\sum_{k=1}^K r_kQ_{t,k}
\right)
\mathbb P(\mathscr G_t\cap\{J_t=0\}\mid \mathcal H_t).
\end{equation}

Since $Vc_t\ge 0$, it follows that
\begin{equation}\label{eq: drift_only_good_non_exceptional_weighted}
\mathbb E\!\left[\Delta_t\mathds{1}\{\mathscr G_t\}\mathds{1}\{J_t=0\}\mid \mathcal H_t\right]
\le
\left(
A_\varepsilon-\frac{\varepsilon}{2}\sum_{k=1}^K r_kQ_{t,k}
\right)
\mathbb P(\mathscr G_t\cap\{J_t=0\}\mid \mathcal H_t).
\end{equation}

Now observe that
$$
\sum_{k=1}^K r_kQ_{t,k}
= \sum_{k=1}^K \sqrt{r_k}\,(\sqrt{r_k}Q_{t,k})
\ge \sqrt{r_{\min}}\sum_{k=1}^K \sqrt{r_k}Q_{t,k}
\ge \sqrt{r_{\min}} \left(\sum_{k=1}^K r_kQ_{t,k}^2\right)^{1/2}
= \sqrt{r_{\min}}\,Z_t.
$$
The third relation is true since each element of the sum $\sum_k \sqrt{r_k} Q_{t,k}$ is non-negative.

Hence, from \eqref{eq: drift_only_good_non_exceptional_weighted},
\begin{equation}\label{eq: drift_only_good_non_exceptional_weighted_Z}
\mathbb E\!\left[\Delta_t\mathds{1}\{\mathscr G_t\}\mathds{1}\{J_t=0\}\mid \mathcal H_t\right]
\le
\left(
A_\varepsilon-\frac{\varepsilon\sqrt{r_{\min}}}{2}Z_t
\right)
\mathbb P(\mathscr G_t\cap\{J_t=0\}\mid \mathcal H_t).
\end{equation}

Since
$$
\Delta_t=\frac12\,\mathbb E[Z_{t+1}^2-Z_t^2\mid \mathcal F_t],
$$
the tower property yields
$$
\mathbb E\!\left[(Z_{t+1}^2-Z_t^2)\mathds{1}\{\mathscr G_t\}\mathds{1}\{J_t=0\}\mid \mathcal H_t\right]
=
2\,\mathbb E\!\left[\Delta_t\mathds{1}\{\mathscr G_t\}\mathds{1}\{J_t=0\}\mid \mathcal H_t\right].
$$
Therefore, by \eqref{eq: drift_only_good_non_exceptional_weighted_Z},
\begin{equation}\label{eq: squared_drift_good_non_exceptional}
\mathbb E\!\left[(Z_{t+1}^2-Z_t^2)\mathds{1}\{\mathscr G_t\}\mathds{1}\{J_t=0\}\mid \mathcal H_t\right]
\le
\left(
2A_\varepsilon-\varepsilon\sqrt{r_{\min}}\,Z_t
\right)
\mathbb P(\mathscr G_t\cap\{J_t=0\}\mid \mathcal H_t).
\end{equation}

Whenever $U_t\ge \theta_\varepsilon$, we have $U_t\le Z_t\mathds{1}\{\mathscr G_t\}\le Z_t$, hence necessarily $\mathscr G_t$ holds and $Z_t\ge \theta_\varepsilon$. By the definition of $\theta_\varepsilon$ (Equation \ref{eq: theta_epsilon}),
$$
\varepsilon\sqrt{r_{\min}}\,Z_t \ge 8A_\varepsilon,
$$
and therefore
$$
2A_\varepsilon-\varepsilon\sqrt{r_{\min}}\,Z_t
\le
-\frac{\varepsilon\sqrt{r_{\min}}}{4}Z_t
=
-2\eta_\varepsilon Z_t.
$$
Hence, on $\{U_t\ge \theta_\varepsilon\}$,
\begin{equation}\label{eq: squared_drift_good_non_exceptional_negative}
\mathbb E\!\left[(Z_{t+1}^2-Z_t^2)\mathds{1}\{\mathscr G_t\}\mathds{1}\{J_t=0\}\mid \mathcal H_t\right]
\le
-2\eta_\varepsilon Z_t\,\mathbb P(\mathscr G_t\cap\{J_t=0\}\mid \mathcal H_t).
\end{equation}

Since
$$
Z_t^2-2\eta_\varepsilon Z_t \le (Z_t-\eta_\varepsilon)^2,
$$
it follows from \eqref{eq: squared_drift_good_non_exceptional_negative} that, on $\{U_t\ge \theta_\varepsilon\}$,
$$
\mathbb E\!\left[Z_{t+1}^2\mathds{1}\{\mathscr G_t\}\mathds{1}\{J_t=0\}\mid \mathcal H_t\right]
\le
(Z_t-\eta_\varepsilon)^2\,
\mathbb P(\mathscr G_t\cap\{J_t=0\}\mid \mathcal H_t).
$$
By Cauchy--Schwarz,
\begin{align*}
\mathbb E\!\left[Z_{t+1}\mathds{1}\{\mathscr G_t\}\mathds{1}\{J_t=0\}\mid \mathcal H_t\right]^2
&\le
\mathbb P(\mathscr G_t\cap\{J_t=0\}\mid \mathcal H_t)\,
\mathbb E\!\left[Z_{t+1}^2\mathds{1}\{\mathscr G_t\}\mathds{1}\{J_t=0\}\mid \mathcal H_t\right] \\
&\le
(Z_t-\eta_\varepsilon)^2\,
\mathbb P(\mathscr G_t\cap\{J_t=0\}\mid \mathcal H_t)^2.
\end{align*}
Since $Z_{t+1}\ge 0$, we conclude that on $\{U_t\ge \theta_\varepsilon\}$,
\begin{equation}\label{eq: Z_(t+1)-Z_t < -eta}
\mathbb E\!\left[(Z_{t+1}-Z_t)\mathds{1}\{\mathscr G_t\}\mathds{1}\{J_t=0\}\mid \mathcal H_t\right]
\le
-\eta_\varepsilon\,\mathbb P(\mathscr G_t\cap\{J_t=0\}\mid \mathcal H_t).
\end{equation}

\paragraph{Compensated process.}
Since in each slot there is at most one arrival and at most one departure, using inverse triangular inequality we get
$$
Z_{t+1}-Z_t \le \left( \sum_{k=1}^K r_k ( Q_{t+1,k}-Q_{t,k})^2\right)^{1/2} \leq \sqrt{2r_{\max}}.
$$
Now recall that $\delta_r = \sqrt{2r_{\max}}$ as defined in Equation \ref{eq: L,eta,delta,C}.

Define the process
$$
U_t:=Z_t\,\mathds{1}\{\mathscr G_t\} - C_r\sum_{s=1}^{t-1}J_s\,\mathds{1}\{\mathscr G_s\}, \qquad t=1,\dots,T+1,
$$
where $C_r$ is defined in Equation \ref{eq: L,eta,delta,C}.
Since $\bm Q_1=0$, we have $Z_1=0$, and thus
$$
U_1=0.
$$

We claim that Lemma \ref{lemma:hajek_drift} applies to $\{U_t\}_{t\ge 1}$ with respect to the filtration $\{\mathcal H_t\}_{t\ge 1}$.

First,
$$
U_{t+1}-U_t = Z_{t+1}\mathds{1}\{\mathscr G_{t+1}\} - Z_t\mathds{1}\{\mathscr G_t\} - C_rJ_t\mathds{1}\{\mathscr G_t\}.
$$
Since $\mathscr G_{t+1}\subseteq \mathscr G_t$, we have $\mathds{1}\{\mathscr G_{t+1}\}\le \mathds{1}\{\mathscr G_t\}$, so
$$
U_{t+1}-U_t
\le
(Z_{t+1}-Z_t)\mathds{1}\{\mathscr G_t\} - C_rJ_t\mathds{1}\{\mathscr G_t\}
\le
Z_{t+1}-Z_t
\le
\delta_r.
$$
The first condition of Lemma \ref{lemma:hajek_drift} is therefore satisfied.

If $U_t<\theta_\varepsilon$, then trivially
$$
\mathbb E[U_{t+1}-U_t\mid \mathcal H_t]\le \delta_r.
$$
Thus the second condition of Lemma \ref{lemma:hajek_drift} also holds.

Now suppose that $U_t\ge \theta_\varepsilon$. Since $U_t\le Z_t\mathds{1}\{\mathscr G_t\}\le Z_t$, we must have $\mathscr G_t$ and $Z_t\ge \theta_\varepsilon$. Then
\begin{align*}
\mathbb E[U_{t+1}-U_t\mid \mathcal H_t]
\le & \; \mathbb E[(Z_{t+1}-Z_t)\mathds{1}\{\mathscr G_t\}\mathds{1}\{J_t=0\}\mid \mathcal H_t]\\
& + \mathbb E[(Z_{t+1}-Z_t-C_r)\mathds{1}\{\mathscr G_t\}\mathds{1}\{J_t=1\}\mid \mathcal H_t].
\end{align*}

By \eqref{eq: Z_(t+1)-Z_t < -eta},
$$
\mathbb E[(Z_{t+1}-Z_t)\mathds{1}\{\mathscr G_t\}\mathds{1}\{J_t=0\}\mid \mathcal H_t]
\le
-\eta_\varepsilon\,\mathbb P(\mathscr G_t\cap\{J_t=0\}\mid \mathcal H_t).
$$

On the other hand, on $\mathscr G_t\cap\{J_t=1\}$ we simply use $Z_{t+1}-Z_t\le \delta_r$, hence
\begin{align*}
\mathbb E[(Z_{t+1}-Z_t-C_r)\mathds{1}\{\mathscr G_t\}\mathds{1}\{J_t=1\}\mid \mathcal H_t]
&\le
(\delta_r-C_r)\,
\mathbb P(\mathscr G_t\cap\{J_t=1\}\mid \mathcal H_t) \\
&=
-\eta_\varepsilon\,
\mathbb P(\mathscr G_t\cap\{J_t=1\}\mid \mathcal H_t).
\end{align*}
Combining the two estimates yields
$$
\mathbb E[U_{t+1}-U_t\mid \mathcal H_t]
\le
-\eta_\varepsilon.
$$
Therefore the third condition of Lemma \ref{lemma:hajek_drift} is satisfied as well.

We may now apply Lemma \ref{lemma:hajek_drift} with
$$
\theta=\theta_\varepsilon, \qquad \delta=\delta_r, \qquad \eta=\eta_\varepsilon.
$$
It follows that
\begin{equation}\label{eq: bound on U_(T+1)}
    \mathbb E[U_{T+1}] \le M_\varepsilon
\end{equation}
with
\begin{equation}\label{eq: M_epsilon}
M_\varepsilon = \left(\frac{8\sum_{k=1}^K\lambda_k c_k^\varepsilon}{\varepsilon \sqrt{r_{\min}}} \right) \cdot V +\frac{8 r_{\max}}{\varepsilon \sqrt{r_{\min}}}+\sqrt{2 r_{\max}} +\frac{16r_{\max}}{\varepsilon \sqrt{r_{\min}}} \log \left( \frac{1024r_{\max}}{\varepsilon^2 r_{\min}}\right).
\end{equation}
Since we are just interested in the dependency on $T$, we can observe that $M_\varepsilon=\mathcal O(V)$ (the choice of the right value of $V$ is horizon-dependent).

Finally, by definition of $U_{T+1}$,
$$
U_{T+1}
=
Z_{T+1}\mathds{1}\{\mathscr G_T\}
-
C_r\sum_{t=1}^T J_t\mathds{1}\{\mathscr G_t\}.
$$
Hence
$$
Z_{T+1}\mathds{1}\{\mathscr G_T\}
\le
U_{T+1}+C_r\sum_{t=1}^T J_t.
$$
Using \eqref{eq: exceptional_slots_total_bound}, we conclude that
$$
Z_{T+1}\mathds{1}\{\mathscr G_T\}
\le
U_{T+1}+2C_rKL_\varepsilon.
$$
Taking expectations and using Equation \ref{eq: bound on U_(T+1)}, we finally get
$$
\mathbb E[Z_{T+1}\mathds{1}\{\mathscr G_T\}] \le M_\varepsilon+2C_rKL_\varepsilon.
$$
We can observe that $2 C_r K L_\varepsilon = \mathcal{O}(K \log (T/\delta))$.

This proves the lemma.
\hfill $\square$

\subsubsection{Proof of Proposition \ref{prop:R4_bound_weighted}}\label{app: proof of R4 bound}
We decompose
$$
R_4(T)
=
\sum_{k=1}^K r_k\,\mathbb E[Q_{T+1,k}\mathds{1}\{\mathscr G_T\}]
+
\sum_{k=1}^K r_k\,\mathbb E[Q_{T+1,k}\mathds{1}\{\overline{\mathscr G}_T\}].
$$

Since at most one task arrives in each slot and the queues start empty, the total backlog satisfies
$$
\sum_{k=1}^K Q_{T+1,k}\le T.
$$
Hence
$$
\sum_{k=1}^K r_kQ_{T+1,k}
\le
r_{\max}\sum_{k=1}^K Q_{T+1,k}
\le
r_{\max}T,
$$
and therefore
$$
\sum_{k=1}^K r_k\,\mathbb E[Q_{T+1,k}\mathds{1}\{\overline{\mathscr G}_T\}]
\le
r_{\max}T\,\mathbb P(\overline{\mathscr G}_T).
$$

On the good event, by Cauchy--Schwarz,
$$
\sum_{k=1}^K r_kQ_{T+1,k}
=
\sum_{k=1}^K \sqrt{r_k}\,(\sqrt{r_k}Q_{T+1,k})
\le
\left(\sum_{k=1}^K r_k\right)^{1/2}
\left(\sum_{k=1}^K r_kQ_{T+1,k}^2\right)^{1/2}
=
\left(\sum_{k=1}^K r_k\right)^{1/2}Z_{T+1}.
$$
Thus
$$
\sum_{k=1}^K r_k\,\mathbb E[Q_{T+1,k}\mathds{1}\{\mathscr G_T\}]
\le
\left(\sum_{k=1}^K r_k\right)^{1/2}
\mathbb E[Z_{T+1}\mathds{1}\{\mathscr G_T\}].
$$
Applying Lemma \ref{lemma:deterministic_weighted_backlog}, we obtain
$$
\sum_{k=1}^K r_k\,\mathbb E[Q_{T+1,k}\mathds{1}\{\mathscr G_T\}]
\le
\left(\sum_{k=1}^K r_k\right)^{1/2}
\left(M_\varepsilon+2C_rKL_\varepsilon\right).
$$

Combining the good-event and bad-event contributions yields
$$
R_4(T)
\le
\left(\sum_{k=1}^K r_k\right)^{1/2}
\left(M_\varepsilon+2C_rKL_\varepsilon\right)
+
r_{\max}T\,\mathbb P(\overline{\mathscr G}_T).
$$

Finally, since $M_\varepsilon=\mathcal O(V)$ and $2C_rKL_\varepsilon = \mathcal{O}(K\log (T/\delta))$, we conclude that
$$
R_4(T)=\mathcal{O}\left(V+K\log (T/\delta)\right) +r_{\max}T\,\mathbb P(\overline{\mathscr G}_T).
$$
\hfill $\square$

\subsection{Bounding $R_1^\mu(T)$ and $R_1^p(T)$}

Before proceeding to the bound on $R_1^\mu(T)$ and $R_1^p(T)$, we need an additional technical lemma. The result extends the backlog estimate of Lemma~\ref{lemma:deterministic_weighted_backlog} to suitable random times, which naturally arise in the analysis of the estimation errors.

For every $k\in[K]$ and every $n\in[T]$, let us first define the random time of the $n$-th effective service on queue $k$ as
\begin{equation}\label{eq: definition tau_kn}
    \tau_{k,n}:=\inf\{t\ge 1:\ N_k(t)=n\}.
\end{equation}
Similarly, for every $k\in[K]$ and every $n\in[T]$, define the random time of the $n$-th informative chatbot exposure on type $k$ as
\begin{equation}\label{eq: definition sigma_kn}
    \sigma_{k,n}:=\inf\{t\ge 1:\ M_k(t)=n\}.
\end{equation}

\begin{lemma}\label{lemma:weighted_random_times}
For every $k\in[K]$ and every $n\in[T]$, the random times $\tau_{k,n}\wedge T$ and $\sigma_{k,n}\wedge T$ are stopping times. Moreover, for every random time $\kappa_{k,n}\in\{\tau_{k,n},\sigma_{k,n}\}$, it holds that
$$
\mathbb E\!\left[
r_kQ_{\kappa_{k,n},k}\,
\mathds 1\{\kappa_{k,n}\le T\}\,
\mathds 1\{\mathscr G_{\kappa_{k,n}}\}
\right] = \mathcal{O}(V + K\log(T/\delta))
$$
\end{lemma}

\begin{proof}

Fix $k\in[K]$ and $n\in[T]$, and let $\kappa_{k,n}\in\{\tau_{k,n},\sigma_{k,n}\}$ as defined in Equations \eqref{eq: definition tau_kn} and \eqref{eq: definition sigma_kn}.

First, $\kappa_{k,n}\wedge T$ is a stopping time. Indeed,
$$
\{\tau_{k,n}\le t\}=\{N_k(t)\ge n\}\in\mathcal F_t, \qquad \{\sigma_{k,n}\le t\}=\{M_k(t)\ge n\}\in\mathcal F_t.
$$
Define the first bad time as
$$
\tau_{\mathrm{bad}}:=\inf\{t\ge 1:\overline{\mathscr G}_t\}\wedge (T+1).
$$
We define an auxiliary process $\{V_t\}_{t=1}^{T+1}$ by
$$
V_t=
\begin{cases}
Z_t - C_r\displaystyle\sum_{s=1}^{t-1}J_s, & \text{if } t<\tau_{\mathrm{bad}},\\[2mm]
V_{\tau_{\mathrm{bad}}-1}-\eta_\varepsilon\,(t-\tau_{\mathrm{bad}}+1), & \text{if } t\ge \tau_{\mathrm{bad}}.
\end{cases}
$$
where $\eta_\varepsilon$ and $C_r$ are defined in Equation \ref{eq: L,eta,delta,C}.
Thus, up to the first bad time, the process $V_t$ coincides with the compensated process $U_t$ used in the proof of Lemma~\ref{lemma:deterministic_weighted_backlog}, while after $\tau_{\mathrm{bad}}$ it is continued deterministically with slope $-\eta_\varepsilon$.

\paragraph{Bounded increments of $V_t$.}
Recall that, according to Equation \ref{eq: L,eta,delta,C},
$$
\delta_r=\sqrt{2r_{\max}}, \qquad C_r=\delta_r+\eta_\varepsilon.
$$
Plus recall that, since in each slot there is at most one arrival and at most one departure, we have
$$
|Z_{t+1}-Z_t| \le \left(\sum_{j=1}^K r_j(Q_{t+1,j}-Q_{t,j})^2\right)^{1/2} \le
\delta_r.
$$
Define
\begin{equation}\label{eq: B_epsilon}
    B_\varepsilon:=\delta_r+C_r.
\end{equation}

If $t<\tau_{\mathrm{bad}}-1$, then
$$
V_{t+1}-V_t=Z_{t+1}-Z_t-C_rJ_t,
$$
hence
$$
|V_{t+1}-V_t| \le |Z_{t+1}-Z_t|+C_rJ_t \le \delta_r+C_r = B_\varepsilon.
$$
If $t=\tau_{\mathrm{bad}}-1$, then by definition
$$
V_{t+1}-V_t=-\eta_\varepsilon,
$$
and if $t\ge \tau_{\mathrm{bad}}$, again
$$
V_{t+1}-V_t=-\eta_\varepsilon.
$$
Therefore, for every $t\in[T]$,
\begin{equation}\label{eq: V bounded increment}
|V_{t+1}-V_t|\le B_\varepsilon \qquad \text{a.s.}
\end{equation}

Next we verify the negative drift condition above the threshold $\theta_\varepsilon$, defined in Equation \ref{eq: theta_epsilon}. If $t<\tau_{\mathrm{bad}}-1$, then $V_t$ coincides with the compensated process $U_t$ considered in the proof of Lemma \ref{lemma:deterministic_weighted_backlog}. Hence, whenever $V_t\ge \theta_\varepsilon$, that proof gives
$$
\mathbb E[V_{t+1}-V_t\mid \mathcal F_t]\le -\eta_\varepsilon.
$$
If instead $t\ge \tau_{\mathrm{bad}}-1$, then by construction
$$
V_{t+1}-V_t=-\eta_\varepsilon.
$$
Thus, for every $t\in[T]$,
\begin{equation}\label{eq: V negative drift}
\mathbb E[V_{t+1}-V_t\mid \mathcal F_t]\le -\eta_\varepsilon
\qquad \text{whenever } V_t\ge \theta_\varepsilon.
\end{equation}

\paragraph{Exponential recursion.}
Set
$$
a_\varepsilon:=\frac{\eta_\varepsilon}{B_\varepsilon^2}, \qquad \rho_\varepsilon:=\exp\!\left(-\frac{\eta_\varepsilon^2}{2B_\varepsilon^2}\right)\in(0,1).
$$
Let
$$
\Delta V_t:=V_{t+1}-V_t.
$$
By \eqref{eq: V bounded increment}, we have $|\Delta V_t|\le B_\varepsilon$ almost surely. Therefore, by the conditional Hoeffding lemma, with $-B_\varepsilon \leq \Delta V_t \leq B_\varepsilon$,
$$
\mathbb E\!\left[e^{a_\varepsilon(\Delta V_t-\mathbb E[\Delta V_t\mid\mathcal F_t])}\mid \mathcal F_t\right] \le \exp\!\left(\frac{a_\varepsilon^2B_\varepsilon^2}{2}\right).
$$
Hence, on the event $\{V_t\ge \theta_\varepsilon\}$, using \eqref{eq: V negative drift},
\begin{align*}
\mathbb E[e^{a_\varepsilon\Delta V_t}\mid \mathcal F_t]
&\le \exp\!\left( a_\varepsilon\mathbb E[\Delta V_t\mid\mathcal F_t] +\frac{a_\varepsilon^2B_\varepsilon^2}{2} \right) \\
&\le \exp\!\left( -a_\varepsilon\eta_\varepsilon+\frac{a_\varepsilon^2B_\varepsilon^2}{2} \right) \\
&= \exp\!\left(-\frac{\eta_\varepsilon^2}{2B_\varepsilon^2}\right) =\rho_\varepsilon.
\end{align*}
Therefore,
\begin{equation}\label{eq: exp drift above threshold}
\mathbb E[e^{a_\varepsilon V_{t+1}}\mid\mathcal F_t] \le \rho_\varepsilon e^{a_\varepsilon V_t} \qquad \text{on } \{V_t\ge \theta_\varepsilon\}.
\end{equation}

On the other hand, if $V_t<\theta_\varepsilon$, then using \eqref{eq: V bounded increment},
$$
V_{t+1}\le V_t+B_\varepsilon<\theta_\varepsilon+B_\varepsilon,
$$
and hence
\begin{equation}\label{eq: exp drift below threshold}
e^{a_\varepsilon V_{t+1}} \le e^{a_\varepsilon(\theta_\varepsilon+B_\varepsilon)}
\qquad \text{on } \{V_t<\theta_\varepsilon\}.
\end{equation}

Combining \eqref{eq: exp drift above threshold} and \eqref{eq: exp drift below threshold}, we obtain
$$
\mathbb E[e^{a_\varepsilon V_{t+1}}\mid\mathcal F_t] \le \rho_\varepsilon e^{a_\varepsilon V_t} + e^{a_\varepsilon(\theta_\varepsilon+B_\varepsilon)}.
$$
Let
$$
M_t:=\mathbb E[e^{a_\varepsilon V_t}].
$$
Taking expectations, we get the recursion
$$
M_{t+1}\le \rho_\varepsilon M_t+e^{a_\varepsilon(\theta_\varepsilon+B_\varepsilon)}.
$$
Since $V_1=Z_1=0$, we have $M_1=1$. Iterating the recursion yields and using the fact that $\rho_\varepsilon \leq 1$
$$
M_t \le \rho_\varepsilon^{\,t-1} + \frac{1-\rho_\varepsilon^{\,t-1}}{1-\rho_\varepsilon} e^{a_\varepsilon(\theta_\varepsilon+B_\varepsilon)}
\le 1+\frac{e^{a_\varepsilon(\theta_\varepsilon+B_\varepsilon)}}{1-\rho_\varepsilon}.
$$
Define
$$
b_\varepsilon := e^{-a_\varepsilon\theta_\varepsilon}
\left( 1+\frac{e^{a_\varepsilon(\theta_\varepsilon+B_\varepsilon)}}{1-\rho_\varepsilon} \right).
$$
Then, by Markov's inequality, for every $t\in[T+1]$ and every $x\ge 0$,
\begin{equation}\label{eq: V_exp_tail}
\begin{aligned}
\mathbb P(V_t\ge \theta_\varepsilon+x) &= \mathbb P\!\left(e^{a_\varepsilon V_t}\ge e^{a_\varepsilon(\theta_\varepsilon+x)}\right)\\
&\le e^{-a_\varepsilon(\theta_\varepsilon+x)}\,\mathbb E[e^{a_\varepsilon V_t}]\\
&\le b_\varepsilon e^{-a_\varepsilon x}.
\end{aligned}
\end{equation}

\paragraph{Bound on the maximum.}
Applying a union bound over $t=1,\dots,T$, from \eqref{eq: V_exp_tail} we obtain
\begin{equation}\label{eq: union bound on V_s}
\mathbb P\!\left(\max_{1\le s\le T}V_s\ge \theta_\varepsilon+x\right) \le Tb_\varepsilon e^{-a_\varepsilon x}.
\end{equation}
This is true since the probability of the maximum is equal to probability that at least one $V_s$ is greater than the threshold, i.e. the union.

Integrating the bound, we obtain
\begin{align}
\mathbb E\!\left[\max_{1\le s\le T}V_s\right] &= \int_0^\infty \mathbb P\!\left(\max_{1\le s\le T}V_s\ge y\right)\,dy \notag\\
&\le \theta_\varepsilon + \int_0^\infty \min\{1,Tb_\varepsilon e^{-a_\varepsilon x}\}\,dx.
\end{align}
Indeed, for $y\le \theta_\varepsilon$ we use the trivial bound by $1$, while for
$y=\theta_\varepsilon+x$, $x\ge0$, we use Equation \ref{eq: union bound on V_s}.
Let $C:=Tb_\varepsilon$. Since
$$
\min\{1,Ce^{-a_\varepsilon x}\} \le \mathds{1}\left\{x\le \frac{\log C}{a_\varepsilon}\right\} + Ce^{-a_\varepsilon x} \mathds{1}\left\{x> \frac{\log C}{a_\varepsilon}\right\},
$$
we get
\begin{align}
\int_0^\infty \min\{1,Ce^{-a_\varepsilon x}\}\,dx &\le \frac{\log C}{a_\varepsilon}
+ \int_{\frac{\log C}{a_\varepsilon}}^\infty Ce^{-a_\varepsilon x}\,dx \notag\\
&= \frac{\log C}{a_\varepsilon} + \frac{1}{a_\varepsilon}.
\end{align}
Therefore,
\begin{equation}\label{eq: max_V_bound}
\mathbb E\!\left[\max_{1\le s\le T}V_s\right] \le \theta_\varepsilon + \frac{1+\log(Tb_\varepsilon)}{a_\varepsilon}.
\end{equation}

\paragraph{Control at the random time $\kappa_{k,n}$.}
Define the event
$$
A_{k,n}:=\{\kappa_{k,n}\le T\}\cap\{\mathscr G_{\kappa_{k,n}}\}.
$$
On $A_{k,n}$ we necessarily have $\kappa_{k,n}<\tau_{\mathrm{bad}}$, and therefore
$$
V_{\kappa_{k,n}} = Z_{\kappa_{k,n}}-C_r\sum_{s=1}^{\kappa_{k,n}-1}J_s.
$$
Hence, on $A_{k,n}$,
$$
Z_{\kappa_{k,n}} = V_{\kappa_{k,n}}+C_r\sum_{s=1}^{\kappa_{k,n}-1}J_s
\le V_{\kappa_{k,n}}+C_r\sum_{s=1}^{T}J_s.
$$
Using the pathwise bound in Equation \eqref{eq: exceptional_slots_total_bound},
$$
\sum_{s=1}^{T}J_s\le 2KL_\varepsilon,
$$
we deduce that
$$
Z_{\kappa_{k,n}}\mathds 1\{A_{k,n}\} \le \left(\max_{1\le s\le T}V_s+2C_rKL_\varepsilon\right)\mathds 1\{A_{k,n}\}
\le \max_{1\le s\le T}V_s+2C_rKL_\varepsilon.
$$

Finally, by the definition of $Z_t$,
$$
r_kQ_{t,k} = \sqrt{r_k}\,(\sqrt{r_k}Q_{t,k}) \le \sqrt{r_k}\,Z_t.
$$
Therefore,
\begin{align*}
r_kQ_{\kappa_{k,n},k}\mathds 1\{A_{k,n}\} &\le \sqrt{r_k}\,Z_{\kappa_{k,n}}\mathds 1\{A_{k,n}\}\\
&\le \sqrt{r_k}\left(\max_{1\le s\le T}V_s+2C_rKL_\varepsilon\right).
\end{align*}
Taking expectations and using \eqref{eq: max_V_bound}, we conclude that
\begin{align*}
\mathbb E\!\left[r_kQ_{\kappa_{k,n},k}\,\mathds 1\{\kappa_{k,n}\le T\}\,
\mathds 1\{\mathscr G_{\kappa_{k,n}}\}\right]&\le\sqrt{r_k}\,\mathbb E\!\left[\max_{1\le s\le T}V_s\right]+2\sqrt{r_k}\,C_rKL_\varepsilon\\
&\le \sqrt{r_k}\left( \theta_\varepsilon+\frac{1+\log(Tb_\varepsilon)}{a_\varepsilon}
+2C_rKL_\varepsilon \right).
\end{align*}
Since $\theta_\varepsilon=\mathcal O(V)$, $a_\varepsilon^{-1}=\mathcal O(1)$, $\log (T b_\varepsilon)=\mathcal O(\log T)$, and $2C_rKL_\varepsilon=\mathcal O(K\log(T/\delta))$, this yields
$$
\mathbb E\!\left[r_kQ_{\kappa_{k,n},k}\, \mathds 1\{\kappa_{k,n}\le T\}\,
\mathds 1\{\mathscr G_{\kappa_{k,n}}\} \right]
\le \sqrt{r_k}\,\mathcal O(V+K\log(T/\delta)).
$$
This proves the lemma.
\end{proof}

\subsubsection{Proof of Proposition \ref{prop:R1_weighted_bounds}}
\label{app: proof of R_1 bounds}

\paragraph{Bound of $R_1^\mu(T)$}
On the event $\mathscr G_t$, for every $k\in[K]$,
$$
0\le \bar\mu_k(t-1)-\mu_k \le 2\beta_k^\mu(t-1).
$$
Therefore,
\begin{equation}\label{eq:R1_mu_first_step}
R_1^\mu(T) \le \frac{2}{V} \mathbb E\!\left[ \sum_{t=1}^T r_{a_t}Q_{t,a_t}\beta^\mu_{a_t}(t-1)\mathds 1\{\mathscr G_t\} \right].
\end{equation}

We now reorganize the sum by type and service number. Whenever queue $k$ is effectively served at time $t$, we have $t=\tau_{k,n}$ for some $n\in[N_k(T)]$ (as defined in Equation \eqref{eq: definition tau_kn}). Hence,
$$
R_1^\mu(T) \le \frac{2}{V} \sum_{k=1}^K \mathbb E\!\left[
\sum_{n=1}^{N_k(T)} r_kQ_{\tau_{k,n},k}\, \beta_k^\mu(\tau_{k,n}-1)\,
\mathds 1\{\mathscr G_{\tau_{k,n}}\} \right].
$$
Extending the sum to $n=1,\dots,T$, we get
\begin{equation}\label{eq:R1_mu_tau_expansion}
R_1^\mu(T) \le \frac{2}{V} \sum_{k=1}^K \sum_{n=1}^T \mathbb E\!\left[
r_kQ_{\tau_{k,n},k}\, \beta_k^\mu(\tau_{k,n}-1)\,
\mathds 1\{\tau_{k,n}\le T\}\, \mathds 1\{\mathscr G_{\tau_{k,n}}\}
\right].
\end{equation}

Now we bound $\beta_k^\mu(\tau_{k,n}-1)$ on the event $\{\tau_{k,n}\le T\}$.

If $n=1$, then $N_k(\tau_{k,1}-1)=N_k(0)=0$, hence
$$
\beta_k^\mu(\tau_{k,1}-1)=\beta_k^\mu(0)=1.
$$
Since $\log(4KT/\delta)\ge \log 4$, we have
$$
1 \le 2\sqrt{\frac{\log(4KT/\delta)}{2}} = 2\sqrt{\frac{\log(4KT/\delta)}{2\cdot 1}}.
$$

If $n\ge 2$, then on $\{\tau_{k,n}\le T\}$ we have $N_k(\tau_{k,n}-1)=n-1\ge 1$, so
$$
\beta_k^\mu(\tau_{k,n}-1) = \sqrt{\frac{\log\!\left(4K(\tau_{k,n}-1)/\delta\right)}{2(n-1)}}
\le \sqrt{\frac{\log(4KT/\delta)}{2((n-1)\vee 1)}}.
$$
Using the elementary inequality
$$
\frac{1}{\sqrt{(n-1)\vee 1}}\le \frac{2}{\sqrt n},
\qquad \forall n\ge 1,
$$
we deduce that
$$
\beta_k^\mu(\tau_{k,n}-1) \le 2\sqrt{\frac{\log(4KT/\delta)}{2n}}.
$$

Therefore, for every $n\in[T]$,
\begin{equation}\label{eq:mu_beta_bound}
\beta_k^\mu(\tau_{k,n}-1) \le 2\sqrt{\frac{\log(4KT/\delta)}{2n}}.
\end{equation}

Substituting \eqref{eq:mu_beta_bound} into \eqref{eq:R1_mu_tau_expansion}, we obtain
$$
R_1^\mu(T) \le \frac{4}{V} \sqrt{\frac{\log(4KT/\delta)}{2}} \sum_{k=1}^K
\sum_{n=1}^T \frac{1}{\sqrt n}\, \mathbb E\!\left[ r_kQ_{\tau_{k,n},k}\,
\mathds 1\{\tau_{k,n}\le T\}\, \mathds 1\{\mathscr G_{\tau_{k,n}}\} \right].
$$

We may now apply Lemma \ref{lemma:weighted_random_times}, which yields
$$
\mathbb E\!\left[ r_kQ_{\tau_{k,n},k}\, \mathds 1\{\tau_{k,n}\le T\}\,
\mathds 1\{\mathscr G_{\tau_{k,n}}\} \right] \le \sqrt{r_k}\,\overline M_\varepsilon.
$$
with
\begin{equation}\label{eq: M_epsilon bar}
\overline M_\varepsilon:=\theta_\varepsilon+\frac{1+\log(Tb_\varepsilon)}{a_\varepsilon} +2C_rKL_\varepsilon
\end{equation}
Hence
\begin{align*}
R_1^\mu(T) &\le \frac{4\,\overline M_\varepsilon}{V} \sqrt{\frac{\log(4KT/\delta)}{2}} \sum_{k=1}^K \sqrt{r_k}
\sum_{n=1}^T \frac{1}{\sqrt n} \\
&\le \frac{8\,\overline M_\varepsilon}{V} \left(\sum_{k=1}^K \sqrt{r_k}\right) \sqrt{T\log\!\left(\frac{4KT}{\delta}\right)},
\end{align*}
where we used
$$
\sum_{n=1}^T \frac{1}{\sqrt n}\le 2\sqrt T.
$$

Finally, using
$$
\sum_{k=1}^K \sqrt{r_k}\le K\sqrt{r_{\max}},
$$
we obtain the simpler bound
$$
R_1^\mu(T) \le \frac{8K\sqrt{r_{\max}}\,\overline M_\varepsilon}{V} \sqrt{T\log\!\left(\frac{4KT}{\delta}\right)}.
$$
This completes the first part of the proof.

\paragraph{Bound of $R_1^p(T)$}
The same reasoning holds in the bound of $R_1^p(T)$ as well.
On the event $\mathscr G_t$, for every $k\in[K]$,
$$
0\le \bar p_k(t-1)-p_k \le 2\beta_k^p(t-1,\delta).
$$
Therefore,
\begin{equation}\label{eq:R1_p_first_step}
R_1^p(T) \le \frac{2}{V} \mathbb E\!\left[ \sum_{t=1}^T r_{X_t}Q_{t,X_t}c_t\beta^p_{X_t}(t-1)\mathds 1\{\mathscr G_t\} \right].
\end{equation}

Since the chatbot decision is bang-bang, we have $c_t\in\{0,1\}$. Thus only the informative slots with $c_t=1$ contribute to the sum. Whenever $X_t=k$ and $c_t=1$, we have $t=\sigma_{k,n}$ (as defined in Equation \eqref{eq: definition sigma_kn}) for some $n\in[M_k(T)]$. Hence,
$$
R_1^p(T) \le \frac{2}{V} \sum_{k=1}^K \mathbb E\!\left[
\sum_{n=1}^{M_k(T)} r_kQ_{\sigma_{k,n},k}\, \beta_k^p(\sigma_{k,n}-1)\, \mathds 1\{\mathscr G_{\sigma_{k,n}}\}
\right].
$$
Extending the sum to $n=1,\dots,T$, we get
\begin{equation}\label{eq:R1_p_sigma_expansion}
R_1^p(T) \le \frac{2}{V} \sum_{k=1}^K \sum_{n=1}^T
\mathbb E\!\left[ r_kQ_{\sigma_{k,n},k}\, \beta_k^p(\sigma_{k,n}-1)\,
\mathds 1\{\sigma_{k,n}\le T\}\, \mathds 1\{\mathscr G_{\sigma_{k,n}}\}
\right].
\end{equation}

For bounding $\beta_k^p(\sigma_{k,n}-1)$ on the event $\{\sigma_{k,n}\le T\}$, we can apply the same exact reasoning used for the bound of $\beta_k^\mu(\tau_{k,n}-1)$ and we get

$$
\beta_k^p(\sigma_{k,n}-1) \le 2\sqrt{\frac{\log(4KT/\delta)}{2n}}.
$$

Therefore, for every $n\in[T]$,
\begin{equation}\label{eq:p_beta_bound}
\beta_k^p(\sigma_{k,n}-1) \le 2\sqrt{\frac{\log(4KT/\delta)}{2n}}.
\end{equation}

Substituting \eqref{eq:p_beta_bound} into \eqref{eq:R1_p_sigma_expansion}, we obtain
$$
R_1^p(T) \le \frac{4}{V} \sqrt{\frac{\log(4KT/\delta)}{2}} \sum_{k=1}^K
\sum_{n=1}^T \frac{1}{\sqrt n}\, \mathbb E\!\left[ r_kQ_{\sigma_{k,n},k}\,
\mathds 1\{\sigma_{k,n}\le T\}\, \mathds 1\{\mathscr G_{\sigma_{k,n}}\}
\right].
$$

Applying Lemma \ref{lemma:weighted_random_times} we get
$$
\mathbb E\!\left[r_kQ_{\sigma_{k,n},k}\, \mathds 1\{\sigma_{k,n}\le T\}\,
\mathds 1\{\mathscr G_{\sigma_{k,n}}\} \right] \le \sqrt{r_k}\,\overline M_\varepsilon.
$$
with $\overline M_\varepsilon$ defined in Equation \eqref{eq: M_epsilon bar}.
Hence
\begin{align*}
R_1^p(T) &\le \frac{4\,\overline M_\varepsilon}{V} \sqrt{\frac{\log(4KT/\delta)}{2}} \sum_{k=1}^K \sqrt{r_k}
\sum_{n=1}^T \frac{1}{\sqrt n} \\
&\le \frac{8\,\overline M_\varepsilon}{V} \left(\sum_{k=1}^K \sqrt{r_k}\right) \sqrt{T\log\!\left(\frac{4KT}{\delta}\right)},
\end{align*}
where we used again
$$
\sum_{n=1}^T \frac{1}{\sqrt n}\le 2\sqrt T.
$$

Using again
$$
\sum_{k=1}^K \sqrt{r_k}\le K\sqrt{r_{\max}},
$$
we obtain the bound
$$
R_1^p(T) \le \frac{8K\sqrt{r_{\max}}\,\overline M_\varepsilon}{V} \sqrt{T\log\!\left(\frac{4KT}{\delta}\right)}.
$$
This completes the second part of the proof.

\paragraph{Conclusion.}
Combining the two bounds, we obtain
$$
R_1^\mu(T),\; R_1^p(T) \le CK\frac{\overline M_\varepsilon}{V} \sqrt{T\log(4KT/\delta)}
$$
for some constant $C>0$. Since $\overline M_\varepsilon=\mathcal O(V),$
it follows that
$$
R_1^\mu(T),\; R_1^p(T) = \mathcal O\!\left( K\sqrt{T\log(T/\delta)} \right).
$$
This concludes the proof.
\hfill $\square$

\subsection{Proof of Theorem \ref{thm: main regret bound}}
\label{app: final bound on the regret}

Recall from Lemma~\ref{lemma: regret decomposition} that
$$
R_T^{\texttt{UCB-DPP}} \le R_1^\mu(T)+R_1^p(T)+R_2(T)+R_3(T)+R_4(T).
$$

We first make explicit the constants appearing in the bound on the estimation terms.
Recall that
$$
\eta_\varepsilon = \frac{\varepsilon\sqrt{r_{\min}}}{8}, \qquad
\delta_r=\sqrt{2r_{\max}}, \qquad C_r=\delta_r+\eta_\varepsilon,
$$
and
$$
B_\varepsilon = \delta_r+C_r = 2\sqrt{2r_{\max}} + \frac{\varepsilon\sqrt{r_{\min}}}{8}.
$$
Moreover,
$$
a_\varepsilon = \frac{\eta_\varepsilon}{B_\varepsilon^2} = \frac{\varepsilon\sqrt{r_{\min}}}{8\left(2\sqrt{2r_{\max}}+\frac{\varepsilon\sqrt{r_{\min}}}{8}\right)^2},
$$
and
$$
\rho_\varepsilon = \exp\!\left(-\frac{\eta_\varepsilon^2}{2B_\varepsilon^2} \right).
$$
The constant $b_\varepsilon$ is defined as
$$
b_\varepsilon = e^{-a_\varepsilon\theta_\varepsilon} \left(
1+ \frac{e^{a_\varepsilon(\theta_\varepsilon+B_\varepsilon)}}{1-\rho_\varepsilon} \right).
$$
Equivalently,
$$
b_\varepsilon = e^{-a_\varepsilon\theta_\varepsilon}
+ \frac{e^{a_\varepsilon B_\varepsilon}}{1-\rho_\varepsilon}.
$$
Since $\theta_\varepsilon\ge 0$, we have
$$
b_\varepsilon \le 1+ \frac{e^{a_\varepsilon B_\varepsilon}}{1-\rho_\varepsilon}.
$$
Thus, defining
$$
\Gamma_\varepsilon := \log\left( 1+ \frac{e^{a_\varepsilon B_\varepsilon}}{1-\rho_\varepsilon} \right),
$$
we have
$$
\log b_\varepsilon\le \Gamma_\varepsilon.
$$
Notice that $\Gamma_\varepsilon$ depends only on
$\varepsilon,r_{\min},r_{\max}$, and is independent of $T,V$ and $K$.

Now recall that
$$
\theta_\varepsilon = \frac{8}{\varepsilon\sqrt{r_{\min}}}\left(
r_{\max}+V\sum_{k=1}^K\lambda_k c_k^\varepsilon \right),
$$
and
$$
L_\varepsilon = \left\lceil \frac{32\log(4KT/\delta)}{\varepsilon^2} \right\rceil.
$$
Therefore,
$$
\overline M_\varepsilon = \theta_\varepsilon + \frac{1+\log(Tb_\varepsilon)}{a_\varepsilon} + 2C_rKL_\varepsilon
$$
satisfies the explicit upper bound
$$
\overline M_\varepsilon \le \theta_\varepsilon + \frac{1+\log T+\Gamma_\varepsilon}{a_\varepsilon} + 2C_rK \left\lceil \frac{32\log(4KT/\delta)}{\varepsilon^2} \right\rceil.
$$
Let us denote the right-hand side by
$$
\widetilde M_\varepsilon := \theta_\varepsilon + \frac{1+\log T+\Gamma_\varepsilon}{a_\varepsilon} + 2C_rK
\left\lceil \frac{32\log(4KT/\delta)}{\varepsilon^2} \right\rceil.
$$

By Proposition~\ref{prop:R1_weighted_bounds}, we have
$$
R_1^\mu(T)+R_1^p(T) \le \frac{16K\sqrt{r_{\max}}}{V} \widetilde M_\varepsilon \sqrt{T\log(4KT/\delta)}.
$$

Moreover,
$$
R_2(T) = T\mathbb P(\overline{\mathscr G}_T) \le \delta T^2,
$$
and
$$
R_3(T) = \frac{r_{\max}T}{V}.
$$
Finally, Proposition~\ref{prop:R4_bound_weighted} yields
$$
R_4(T) \le \sqrt{\sum_{k=1}^K r_k} \left( M_\varepsilon+2C_rKL_\varepsilon
\right) + r_{\max}T\mathbb P(\overline{\mathscr G}_T),
$$
where
$$
M_\varepsilon = \frac{8V}{\varepsilon\sqrt{r_{\min}}} \sum_{k=1}^K\lambda_k c_k^\varepsilon + \frac{8r_{\max}}{\varepsilon\sqrt{r_{\min}}} + \sqrt{2r_{\max}}
+ \frac{16r_{\max}}{\varepsilon\sqrt{r_{\min}}} \log\!\left(
\frac{1024r_{\max}} {\varepsilon^2r_{\min}} \right).
$$

We now choose
$$
V=\sqrt T, \qquad \delta=T^{-2}.
$$
Then
$$
\log(4KT/\delta)=\log(4KT^3),
$$
and Proposition~\ref{propos: bound on bad event proba} gives
$$
\mathbb P(\overline{\mathscr G}_T) \le \delta T = T^{-1}.
$$
Thus
$$
R_2(T)\le 1, \qquad r_{\max}T\mathbb P(\overline{\mathscr G}_T)\le r_{\max}.
$$

With this choice of parameters,
$$
\theta_\varepsilon = \frac{8}{\varepsilon\sqrt{r_{\min}}} \left(
r_{\max} + \sqrt T\sum_{k=1}^K\lambda_k c_k^\varepsilon \right),
$$
and
$$
\widetilde M_\varepsilon := \frac{8}{\varepsilon\sqrt{r_{\min}}}\left(
r_{\max} + \sqrt T\sum_{k=1}^K\lambda_k c_k^\varepsilon \right)
+ \frac{1+\log T+\Gamma_\varepsilon}{a_\varepsilon}
+ 2C_rK \left\lceil \frac{32\log(4KT^3)}{\varepsilon^2} \right\rceil.
$$

Combining all the terms, we obtain the explicit bound
\begin{align*}
R_T^{\texttt{UCB-DPP}} \le\;& 16K\sqrt{r_{\max}}\, \widetilde M_\varepsilon \sqrt{\log(4KT^3)} + r_{\max}\sqrt T + 1 + r_{\max}
\\ & + \sqrt{\sum_{k=1}^K r_k} \Bigg[ \frac{8\sqrt T}{\varepsilon\sqrt{r_{\min}}} \sum_{k=1}^K\lambda_k c_k^\varepsilon
+ \frac{8r_{\max}}{\varepsilon\sqrt{r_{\min}}} + \sqrt{2r_{\max}} \\
&+ \frac{16r_{\max}}{\varepsilon\sqrt{r_{\min}}}\log\!\left(
\frac{1024r_{\max}}{\varepsilon^2r_{\min}}\right)
+ 2\left(\sqrt{2r_{\max}}+\frac{\varepsilon\sqrt{r_{\min}}}{8}\right)
K\left\lceil\frac{32\log(4KT^3)}{\varepsilon^2}\right\rceil\Bigg].
\end{align*}

This expression keeps all constants explicit. Since
$$
\sum_{k=1}^K\lambda_k c_k^\varepsilon \le \sum_{k=1}^K \lambda_k =  1,
$$
we have
$$
\widetilde M_\varepsilon = \mathcal O\!\left( \sqrt T+K\log(KT)\right),
$$
where the hidden constants depend only on $\varepsilon,r_{\min},r_{\max}$.
Consequently,
$$
R_T^{\texttt{UCB-DPP}} = \mathcal O\!\left(K\sqrt{T\log T} \right),
$$
up to constants depending on $\varepsilon,r_{\min},r_{\max}$.
\hfill $\square$

\section{Simulations}\label{app: simulations}
In this section, we report numerical simulations of the \texttt{UCB-DPP} policy on synthetic instances.

\begin{figure}[H]
    \centering
    \begin{subfigure}{0.48\textwidth}
        \centering
        \includegraphics[width=\textwidth]{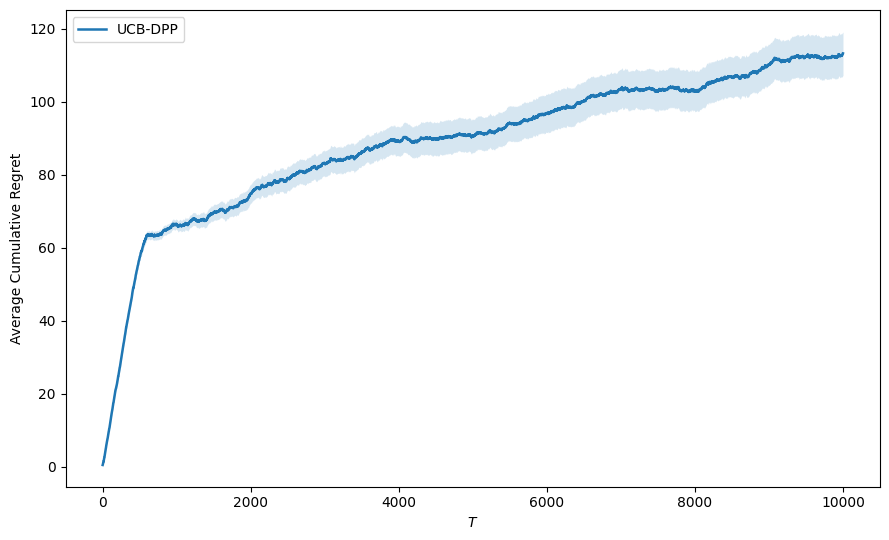}
        \caption{$\bm p^{\mathrm{strong}}$ and $\bm \mu^{\mathrm{weak}}$}
        \label{fig:SC-WH}
    \end{subfigure}
    \hfill
    \begin{subfigure}{0.48\textwidth}
        \centering
        \includegraphics[width=\textwidth]{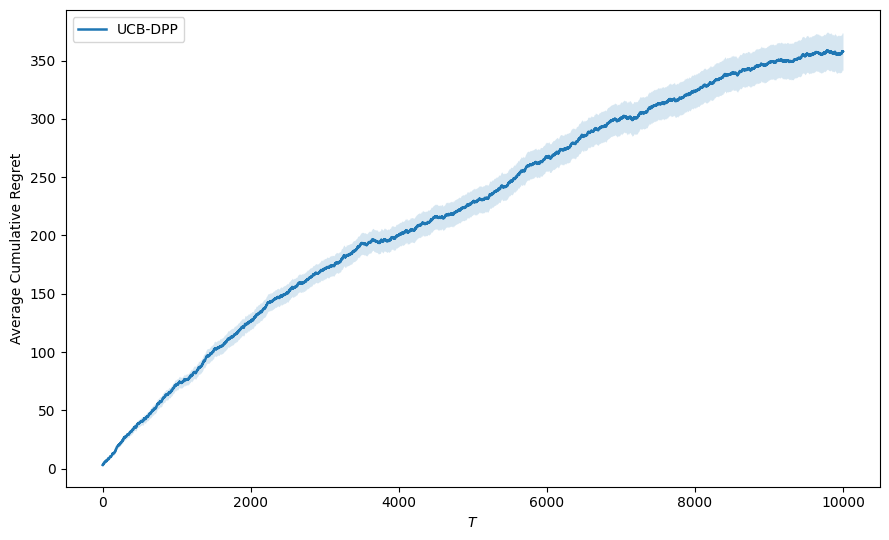}
        \caption{$\bm p^{\mathrm{weak}}$ and $\bm \mu^{\mathrm{strong}}$}
        \label{fig:WC-SH}
    \end{subfigure}
    \caption{Average cumulative regret of the \texttt{UCB-DPP} policy in two different parameter regimes. The shaded regions correspond to one standard error.}
    \label{fig:two_plots}
\end{figure}

\begin{figure}[H]
    \centering
    \includegraphics[width=0.6\linewidth]{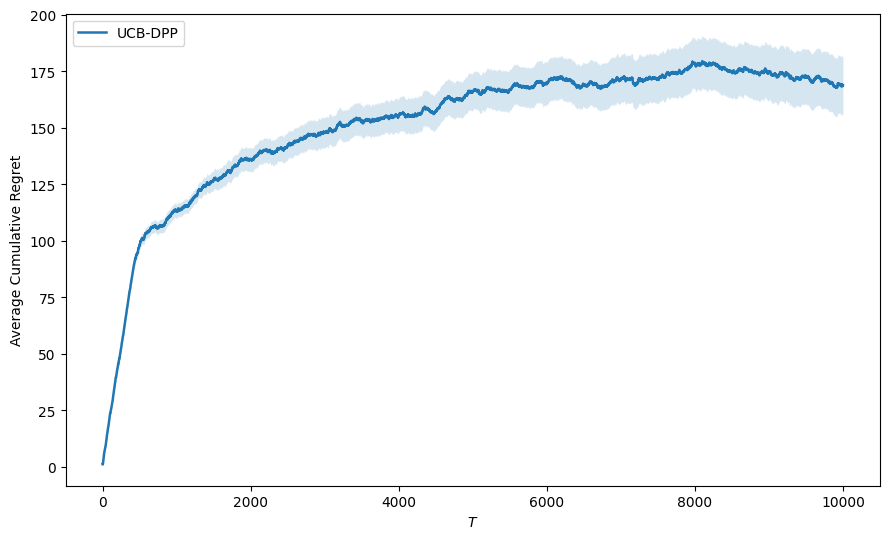}
    \caption{Average cumulative regret of the \texttt{UCB-DPP} policy with $\bm p^{\mathrm{medium}}$ and $\bm \mu^{\mathrm{medium}}$. The shaded region corresponds to one standard error.}
    \label{fig:medium}
\end{figure}

We consider a system with $K=5$ task classes and a uniform arrival distribution, namely $\lambda_k = 1/K$ for all $ k=1,\ldots,K$.
The experiments are carried out under three different regimes for the chatbot success probabilities $\bm p$ and the human service rates $\bm \mu$. The first regime corresponds to a strong chatbot and weak human service (Figure \ref{fig:SC-WH}), the second to a weak chatbot and strong human service (Figure \ref{fig:backlogs}), and the third to an intermediate case in which both components have medium performance (Figure \ref{fig:medium}).

The chatbot success probability vectors are chosen as
$$
\bm p^{\mathrm{strong}}
=
(0.90,0.85,0.95,0.88,0.92),
\qquad
\bm p^{\mathrm{weak}}
=
(0.25,0.30,0.20,0.35,0.28),
$$
and
$$
\bm p^{\mathrm{medium}}
=
(0.55,0.60,0.50,0.65,0.58).
$$
The corresponding human service-rate vectors are
$$
\bm \mu^{\mathrm{strong}}
=
(0.80,0.90,0.85,0.88,0.82),
\qquad
\bm \mu^{\mathrm{weak}}
=
(0.25,0.30,0.35,0.28,0.32),
$$
and
$$
\bm \mu^{\mathrm{medium}}
=
(0.45,0.50,0.48,0.52,0.46).
$$

For each regime, the parameters are chosen so that the static optimization problem is feasible. 

We then solve the static benchmark problem and compute an optimal dual solution $\bm y^*$. The terminal backlog weights $\bm r$ are selected slightly above the corresponding optimal dual multipliers, namely
$$
r_k = 1.05 \cdot y_k^*, \qquad k=1,\ldots,K.
$$
This choice is consistent with the condition required by the static lower-bound result.

The policy parameter is set to $V = \sqrt{T}$ and $\delta = T^{-2}$. In the simulations we use the horizon $T = 10000$.
Each experiment is repeated over $100$ independent runs. The curves report the empirical average cumulative regret, computed with respect to the static benchmark value $T \cdot \operatorname{OPT}(\theta)$. The shaded regions represent one standard error, that is, the empirical standard deviation divided by the square root of the number of runs.

\begin{figure}[H]
    \centering
    \includegraphics[width=0.6\linewidth]{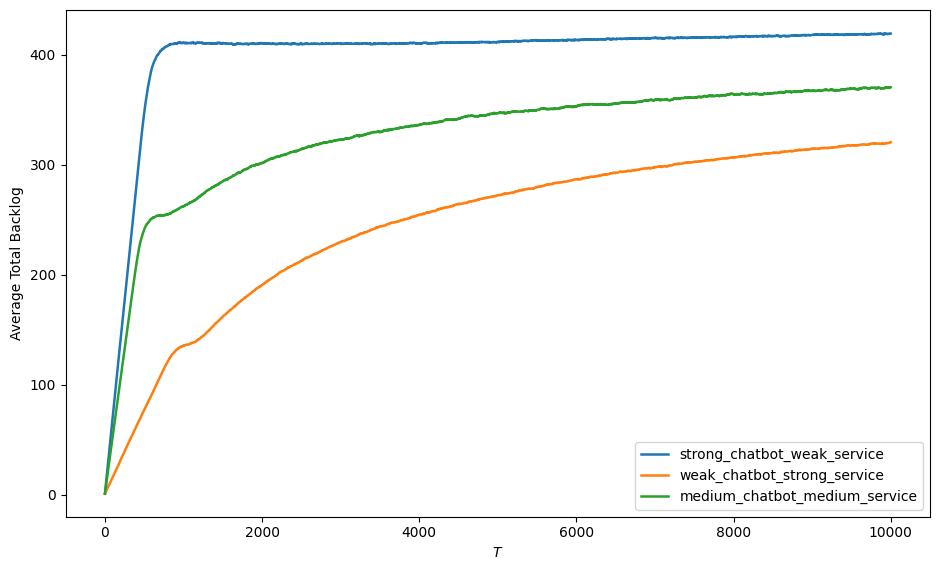}
    \caption{Average total backlog of the \texttt{UCB-DPP} policy in the three parameter regimes.}
    \label{fig:backlogs}
\end{figure}

Figure \ref{fig:backlogs} reports the evolution of the average total backlog,
defined as $\sum_{k=1}^K Q_{T+1,k}$, under the same experimental specifications
described above. The three curves correspond to the strong-chatbot/weak-service,
weak-chatbot/strong-service, and medium-chatbot/medium-service regimes. As in
the regret plots, each curve is averaged over $100$ independent runs,
and the shaded regions represent one standard error.

Overall, the plots show that the cumulative regret remains controlled across all
three regimes, with relatively narrow confidence bands over the replications. The strong-chatbot/weak-service regime exhibits a larger backlog,
as the human server alone has limited capacity and the policy must rely more
heavily on the chatbot to stabilize the system. In contrast, when the human
service rates are high, the queues are more easily drained, leading to a smaller
overall congestion level. The medium regime displays an intermediate behavior.

The backlog plot confirms that the total queue length remains stable over time
in all three scenarios. This is consistent with the stabilizing effect of the Drift-Plus-Penalty scheduling rule.

The code used for the experiments is available at the anonymous link:
\href{https://anonymous.4open.science/r/UCB-DPP-Simulations-ACA5/README.md}{\texttt{UCB-DPP} Simulations}.

\end{document}